\documentclass{article}
\usepackage[dvips]{graphicx}
\usepackage{amssymb,amsmath,color,natbib,amsthm}
\usepackage{url} 
\usepackage{array}
\usepackage{algorithm,algorithmic}

\bibpunct{(}{)}{;}{a}{,}{,}

\title{Convex Relaxation for Combinatorial Penalties}

\author{
Guillaume Obozinski  \\
INRIA - Sierra project-team\\
Laboratoire d'Informatique \\
de l'Ecole Normale Sup\'erieure \\
Paris, France \\
\texttt{guillaume.obozinski@ens.fr}
\and
Francis Bach  \\
INRIA - Sierra project-team\\
Laboratoire d'Informatique \\
de l'Ecole Normale Sup\'erieure \\
Paris, France \\
\texttt{francis.bach@ens.fr} }

\newcommand{\mysec}[1]{Section~\ref{sec:#1}}
\newcommand{\eq}[1]{Eq.~(\ref{eq:#1})}

\def\tw{\textwidth}  
\def\RR{\mathbb{R}}
\def\rb{\mathbb{R}}
\def\VV{\mathcal{V}}

\def\v{ v }
\def\w{ w }
\def\s{ s }
\def\eps{\epsilon}

\def\st{\text{ s.t. }}
\def\supp{\text{Supp}}
\def\vA{\v^A}

\def\OJ{\Omega_J}
\def\OnJ{\Omega^J}

\def\G{\mathcal{G}}
\def\D{\mathcal{D}}

\def\pen{\text{pen}}

\def\ptpow{{\small\circ}}
\def\Fu{F_+}
\def\Fl{F_-}
\def\Fc{\widetilde{F}}
\newcommand\F[1]{F_{[#1]}}
\newcommand\eqdef{:=}
\newcommand\posvspace[1]{}

\newtheorem{example}{Example}
\newtheorem{theorem}{Theorem}
\newtheorem{lemma}{Lemma}
\newtheorem{proposition}{Proposition}

\newtheorem{corollary}[theorem]{Corollary}
\newtheorem{definition}[theorem]{Definition}

\newtheorem*{lemmann}{Lemma}
\newtheorem*{propositionnn}{Proposition}

\newcommand{\BEAS}{\begin{eqnarray*}}
\newcommand{\EEAS}{\end{eqnarray*}}
\newcommand{\BEA}{\begin{eqnarray}}
\newcommand{\EEA}{\end{eqnarray}}
\newcommand{\BEQ}{\begin{equation}}
\newcommand{\EEQ}{\end{equation}}
\newcommand{\BIT}{\begin{itemize}}
\newcommand{\EIT}{\end{itemize}}
\newcommand{\BNUM}{\begin{enumerate}}
\newcommand{\ENUM}{\end{enumerate}}
\newcommand{\lova}{Lov\'asz }

\graphicspath{{./figures_pdf/}}

\begin{document} 
\maketitle
\begin{abstract} 
In this paper, we propose an unifying view of several recently proposed structured sparsity-inducing norms. We consider the situation of a model simultaneously (a) penalized by a set-function defined on the support of the unknown parameter vector which represents prior knowledge on supports, and (b) regularized in $\ell_p$-norm. We show that the natural combinatorial optimization problems obtained may be relaxed into convex optimization problems and 
 introduce a notion, the \emph{lower combinatorial envelope} of a set-function, that characterizes the tightness of our relaxations. We moreover establish links with norms based on latent representations including the latent group Lasso and \emph{block-coding}, and with norms obtained from submodular functions.
\end{abstract}

\section{Introduction}

The last years have seen the emergence of the field of \emph{structured sparsity}, which aims at identifying a model of small complexity given a priori knowledge on its possible structure. 

Various regularizations, in particular convex, have been proposed 
that formalized the notion that prior information can be expressed through
 functions encoding the set of possible or encouraged supports\footnote{By support, we mean the set of indices of non-zero parameters.} in the model. Several convex regularizers for structured sparsity arose as generalizations of the group Lasso \citep{YuaLi06} to the case of overlapping groups \citep{jenatton2011structured,jacob2009group,mairal2011convex}, in particular to tree-structured groups \citep{zhao2009icap,kim3,jenatton2011proximal}. Other formulations have been considered based on variational formulations \citep{Micchelli2011Regularizers}, the perspective of multiple kernel learning \citep{bach2011optim}, submodular functions \citep{bach2010structured} and norms defined as convex hulls \citep{obozinski2011group,chandrasekaran2010convex}.
Non convex approaches include \citet{he2009exploiting,baraniuk2010model,huang2011learning}.
We refer the reader to \citet{huang2011learning} for a concise overview and discussion of the related literature and to \citet{bach2011optim} for a more detailed tutorial presentation.

\label{sec:newnorm}
In this context, and given a model parametrized by a vector of coefficients $w \in \RR^V$ with $V=\{1, \ldots,d\}$, the main objective of this paper is to find an appropriate way to combine together \emph{combinatorial penalties}, that control the structure of a model in terms of the 
sets of variables allowed or favored to enter the function learned, with \emph{continuous regularizers}---such as $\ell_p$-norms, that control the magnitude of their coefficients, into a convex regularization that would control both.

Part of our motivation stems from previous work on regularizers that ``convexify'' combinatorial penalties.  \citet{bach2010structured} proposes to consider the tightest convex relaxation of the restriction of a submodular penalty to a unit $\ell_\infty$-ball in the space of model parameters $w \in \RR^d$. However, this relaxation scheme implicitly assumes that the coefficients are in a unit $\ell_\infty$-ball; 
then, the relaxation obtained induces clustering artifacts of the values of the learned vector. 
It would thus seem desirable to propose relaxation schemes that do not assume that coefficient are bounded but rather to control continuously their magnitude and to find alternatives to the $\ell_\infty$-norm.
Finally the class of functions considered is restricted to submodular functions.

In this paper, we therefore consider combined penalties of the form mentioned above and propose first an appropriate convex relaxation in Section~\ref{sec:cvx_relax}; the properties of general combinatorial functions preserved by the relaxation are captured by the notion of lower combinatorial envelope introduced in Section~\ref{sec:lce}. Section~\ref{sec:lgl} relates the convex regularization obtained to the latent group Lasso and to set-cover penalties, while Section~\ref{sec:examples} provides additional examples, such as the exclusive Lasso. We discuss in more details the case of submodular functions in Section~\ref{sec:submod} and propose for that case efficient algorithms and a theoretical analysis. Finally, we present some experiments in Section~\ref{sec:exp}.

Yet another motivation is to follow loosely the principle of two-part or multiple-part codes  from MDL theory \citep{rissanen1978modeling}. In particular if the model is parametrized by a vector of parameters $w$, it is possible to encode (an approximation of) $w$ itself with a two-part code, by encoding first the support ${\rm Supp}(w)$ ---or set of non-zero values--- of $w$  
with a code length of $F({\rm Supp}(w))$ and by encoding the actual values of $w$ using a code based on a log prior distribution on the vector $w$ that could motivate the choice of an $\ell_p$-norm as a surrogate for the code length.
This leads naturally to consider penalties of the form $\mu F({\rm Supp}(w))+  \nu \|w\|_p^p$ and to find appropriate notions of relaxation.

\paragraph{Notations.} 
When indexing vectors of $\RR^d$ with a set $A$ or $B$ in \emph{exponent}, $x^A$ and $x^B \in \RR^d$ refer to two a priori unrelated vectors; by contrast, when using $A$ as an \emph{index}, and given a vector $x \in \RR^d$, $x_A$ denotes the vector of $\RR^d$ such that  $[x_A]_i=x_i, \, i \in A$ and $[x_A]_i=0, \: i \notin A$.
If $\s$ is a vector in $\RR^d$, we use the shorthand $\s(A) \eqdef \sum_{i \in A} s_i$ and $|\s|$ denotes the vector whose elements are the absolute values $|\s_i|$ of the elements $\s_i$ in $\s$.
For $p\geq 1$, we define $q$ through the relation $\frac{1}{p}+\frac{1}{q}=1$.
The $\ell_q$-norm of a vector $w$ will be noted $\|w\|_q=\big (\sum_{i} w_i^q \big )^{1/q}.$
For a function $f:\RR^d \rightarrow \RR$, we will denote by $f^*$ is Fenchel-Legendre conjugate.
We will write $\overline{\RR}_+$ for $\RR_+ \cup \{+\infty\}.$
\posvspace{-1mm}
\section{Penalties and convex relaxations}
\label{sec:cvx_relax}
\posvspace{-1mm}
Let $V = \{1,\dots,d\}$ and $2^V=\{A \mid A \subset V\}$ its power-set.
We will consider positive-valued set-functions of the form $F: 2^V \rightarrow \overline{\RR}_+$ such that $F(\varnothing)=0$ and $F(A)>0$ for all $A \neq \varnothing$. We do not necessarily assume that $F$ is non-decreasing, even if it would a priori be natural for a penalty function of the support. We however assume that the domain of $F$, defined as $\D_0:=\{A \mid F(A) < \infty\}$, covers $V$, i.e.,  satisfies $\cup_{A \in \D_0} A=V$ (if $F$ is non-decreasing, this just implies that it should be finite on singletons). We will denote by $\iota_{x \in S}$ the indicator function of the set $S$, taking value $0$ on the set and $+\infty$ outside. We will write $[\![k_1,k_2]\!]$ to denote the discrete interval $\{k_1,\ldots,k_2\}$.

With the motivations of the previous section, and denoting by ${\rm Supp}(w)$ the set of non-zero coefficients of a vector $w$, we consider a penalty involving both a \emph{combinatorial} function $F$ and $\ell_p$-regularization:

\begin{equation}
\label{eq:pen}
\pen: w \mapsto \mu\, F({\rm Supp}(w))+  \nu\, \|w\|_p^p,
\end{equation}
where $\mu$ and $\nu$ are positive scalar coefficients.
Since such non-convex discontinuous penalizations are untractable computationally, we undertake to construct an appropriate convex relaxation. The most natural convex surrogate for a non-convex function, say $A$, is arguably its \emph{convex envelope} (i.e., its tightest convex  lower bound) which can be computed as its Fenchel-Legendre bidual $A^{**}$. However, one relatively natural requirement for a regularizer is to ask that it be also \emph{positively homogeneous} (p.h.\ ) since this leads to formulations that are invariant by rescaling of the data.
Our goal will therefore be to construct the tightest positively homogeneous convex lower bound of the penalty considered.

Now, it is a classical result that, given a function $A$, its tightest p.h.\ (but not necessarily convex) lower bound $A_h$ is  $A_h(w) = \inf_{\lambda > 0 } \frac{ A(\lambda w)}{\lambda}$ \citep[see][p.35]{Rockafellar1970Convex}.

This is instrumental here given the following proposition: 
\begin{proposition}
\label{prop:phclb}
Let $A:\RR^d \rightarrow \RR_+$ be a real valued function, $A_h$ defined as above. Then $C$, the tightest positively homogeneous and convex lower bound of $A$, is well-defined and $C=A_h^{**}$.
\end{proposition}
\begin{proof}
The set of convex p.h.\ lower bounds of $A$ is non-empty (since it contains the constant zero function) and stable by taking pointwise maxima. Therefore it has a unique majorant, which we call $C$.
We have for all $w \in \rb^d$,
$ A_h^{\ast \ast} (w) \leqslant C(w) \leqslant A(w)$, by definition of $C$ and the fact that $A_h$ is an p.h.\ lower bound on $A$. We thus have for all $\lambda >0$,
$ A_h^{\ast \ast} (\lambda w) \lambda^{-1}\leqslant C(\lambda w)\lambda^{-1} \leqslant A(\lambda w)\lambda^{-1}$, which implies that for all $w \in \rb^d$, $A_h^{\ast \ast} (w)  \leqslant C(w) \leqslant A_h(w)$. Since $C$ is convex, we must have $C=A_h^{\ast \ast}$, hence the desired result. 
\end{proof}

Using its definition we can easily compute the tightest positively homogeneous  lower bound of the penalization of Eq.~\eqref{eq:pen}, which we denote $\pen_h$:
\BEAS
\pen_h(w)  & = &  \inf_{\lambda > 0 } \frac{\mu}{\lambda} \,F({\rm Supp}(\w)) +  \nu \,\lambda^{p-1}\, \| \w \|_p^p.
\EEAS
Setting the gradient of the objective to $0$, one gets that the minimum is obtained for\\ 
$\displaystyle \lambda =   {\textstyle \big (\frac{\mu q}{\nu p}\big )^{1/p}} \:F({\rm Supp}(\w)) ^{1/p}\: \| w \|_p^{-1},$ and that 
$$\pen_h(w) =  (q \mu)^{1/q} \,(p\nu)^{1/p}\; \Theta(\w),$$
 where we introduced the notation
$$ \Theta(\w) \eqdef \: F({\rm Supp}(\w))^{1/q} \; \| \w\|_p.$$

Up to a constant factor depending on the choices of $\mu$ and $\nu$, we are therefore led to consider the positively homogeneous penalty $\Theta$ we just defined, which combines the two terms multiplicatively.
Consider the norm $\Omega_p$ (or $\Omega_p^F$ if a reference to $F$ is needed) whose dual norm\footnote{The assumptions on the domain $\D_0$ of $F$ and on the positivity of $F$ indeed guarantee that $\Omega_p^*$ is a norm.} is defined as 
\BEA
\Omega_p^*(s):=\max_{A \subset V, A \neq \varnothing} \frac{\|s_A\|_q}{F(A)^{1/q}}.
\EEA
We have the following result:
\begin{proposition}[Convex relaxation]
\label{prop:relax}
The norm $\Omega_p$ is the convex envelope of $\Theta$.
\end{proposition}
\begin{proof}
Denote $\Theta(\w) = \| \w\|_p \, F({\rm Supp}(\w))^{1/q}$, and compute its Fenchel conjugate:
\BEAS
\Theta^\ast(s) & = & \max_{\w \in \RR^d} \w^\top s - \| \w\|_p \, F({\rm Supp}(\w))^{1/q}\\
& = & \max_{A \subset V} \max_{\w_A \in \RR_\ast^{|A|} } \w_A^\top s_A - \| \w_A \|_p \, F(A)^{1/q} \\
& = & \max_{A \subset V}  \iota_{\{  \|s_A\|_{q} \leqslant   F(A)^{1/q} \}} =  \iota_{ \{  \Omega_p^\ast(s) \leqslant 1 \}},
\EEAS
where $\iota_{\{s \in S\}}$ is the indicator of the set $S$, that is the function equal to $0$ on $S$ and $+\infty$ on $S^c$.
The Fenchel bidual of $\Theta$, i.e.,  its largest (thus tightest) convex lower bound, is therefore exactly $\Omega_p$.
\end{proof}
\posvspace{-2mm}
Note that the function $F$ is not assumed submodular in the previous result.
Since the function $\Theta$ depends on $w$ only through $|w|$, by symmetry, the norm $\Omega_p$ is also a function of $|w|$. Given Proposition~\ref{prop:phclb}, we have the immediate corollary:

\begin{corollary}[Two parts-code relaxation]
\label{prop:code2}
Let $p>1$. The norm $w \mapsto (q \mu)^{1/q} (p\nu)^{1/p} \,\Omega_p(w)$ is the tightest convex \emph{positively homogeneous} lower bound of  the function $\w \mapsto  \mu F({\rm Supp}(\w)) + \nu \| \w \|_p^p$.
\end{corollary}

The penalties and relaxation results considered in this section are illustrated on Figure~\ref{fig:penalties}.

\begin{figure}
\begin{tabular}{ccc}
\includegraphics[width=0.3\tw]{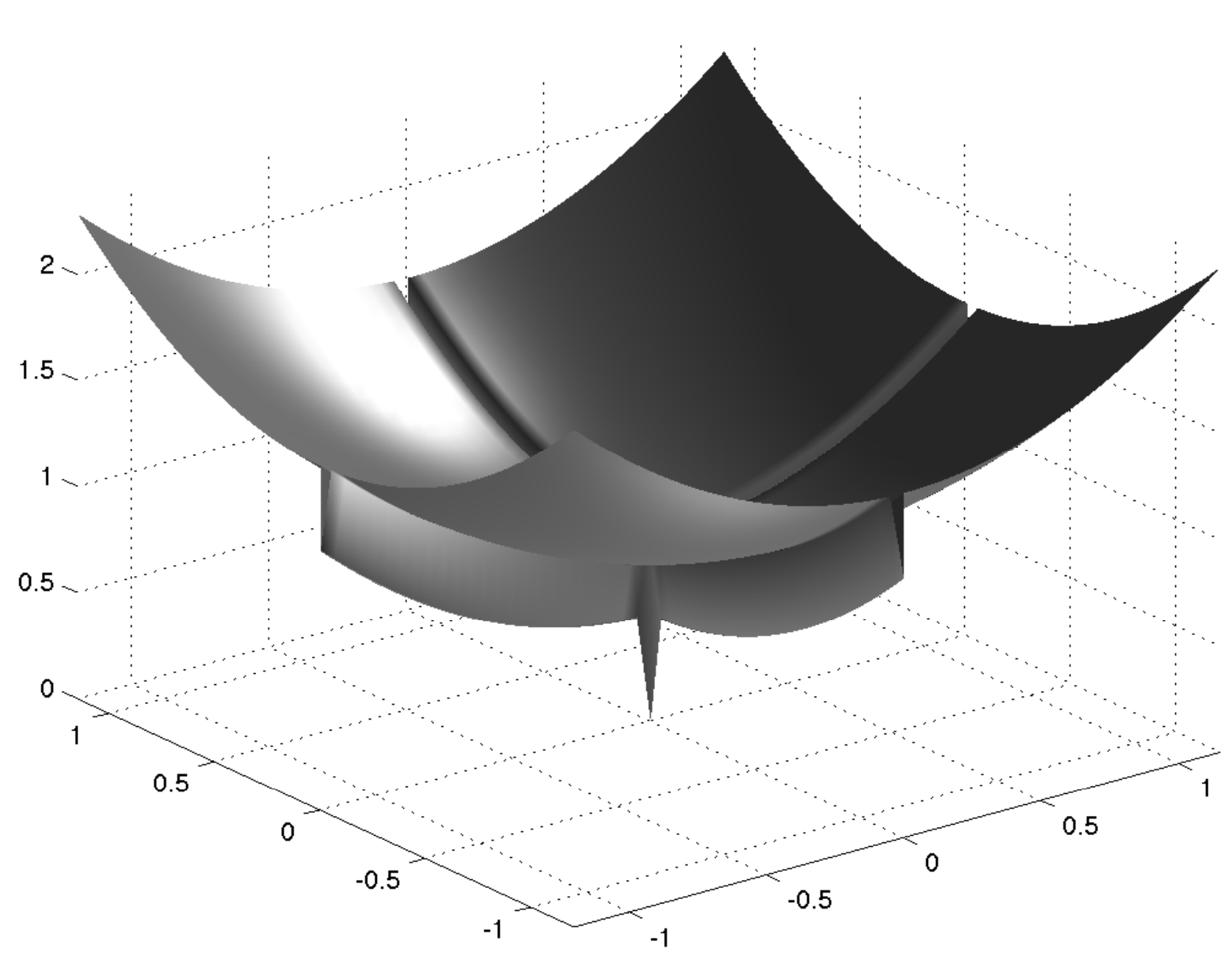} &
\includegraphics[width=0.3\tw]{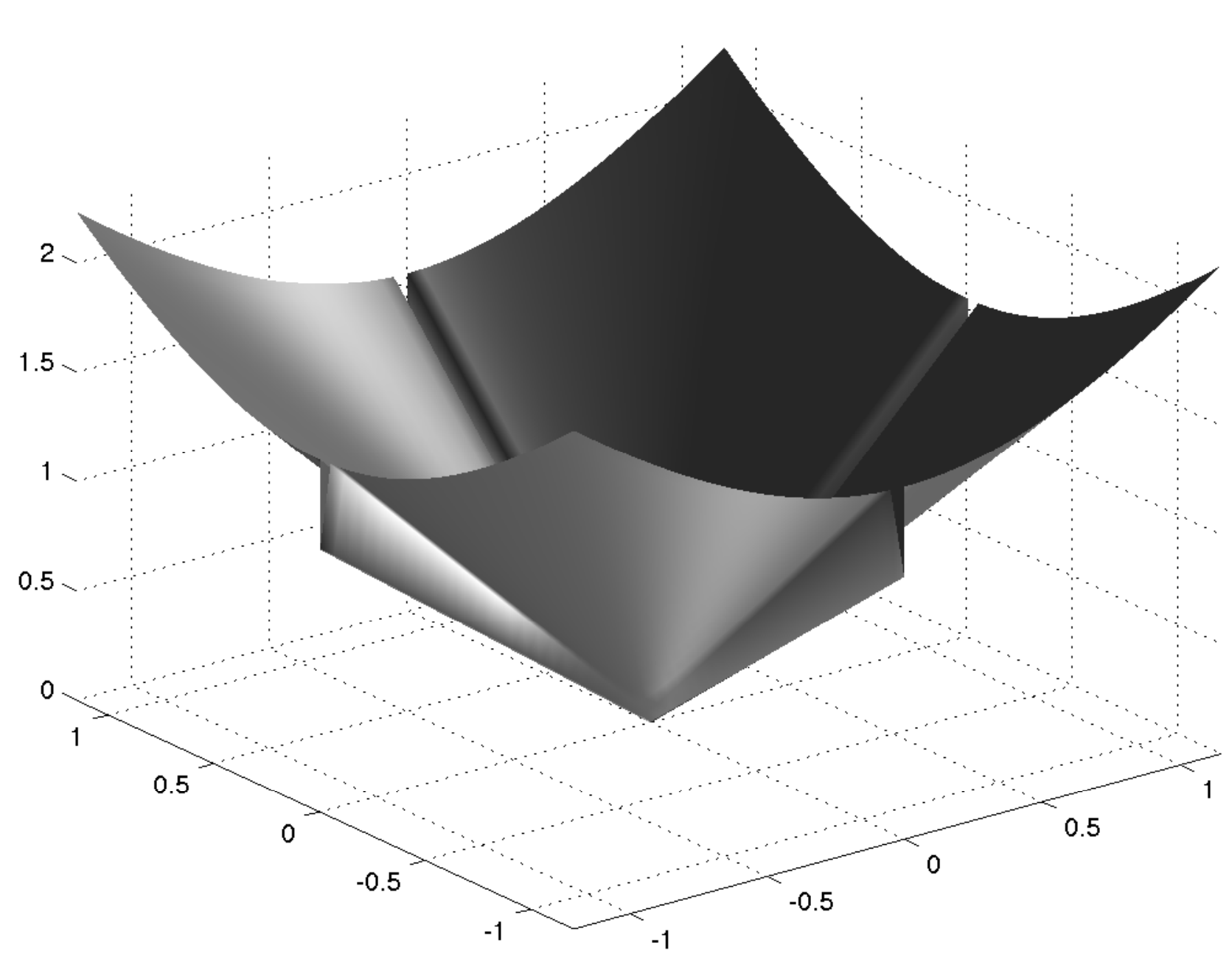}&
\includegraphics[width=0.3\tw]{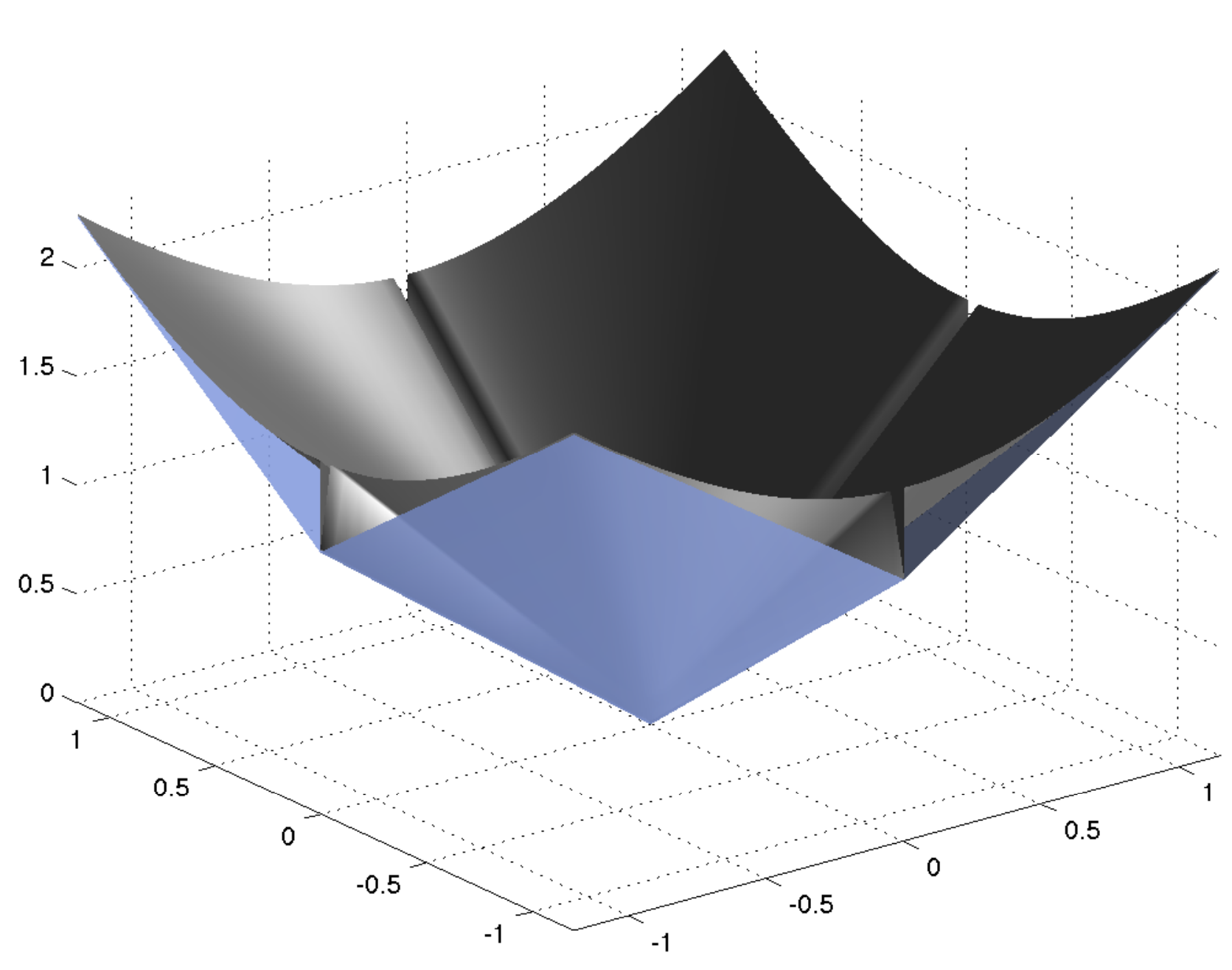}
\end{tabular}
\caption{\textbf{Penalties in 2D} From left to right: the graph of the penalty \texttt{pen}, the graph of penalty $\text{\texttt{pen}}_h$ with $p=2$, and the graph of the norm $\Omega^F_2$ in blue overlaid over graph of $\text{\texttt{pen}}_h$, for the combinatorial function $F:2^V \rightarrow \RR^+$, with $F(\varnothing)=0$, $F(\{1\})=F(\{2\})=1$ and $F(\{1,2\})=1.8$. }
\label{fig:penalties}
\end{figure}

\subsection{Special cases.} 
\label{sec:special}
\paragraph{Case $p=1$.} In that case, letting $d_k=\max_{A \ni k} F(A)$, the dual norm is $\Omega_1^\ast(s) = \max_{k \in V}|s_k|/d_k$
so that $\Omega_1(w) = \sum_{k \in V}  d_k\, | w_k|$ is always a weighted $\ell_1$-norm. 
But regularizing with a weighted $\ell_1$-norm leads to estimators that can potentially have all sparsity patterns possible (even if some are obviously privileged) and in that sense a weighted $\ell_1$-norm cannot encode hard structural constraints on the patterns.
Since this means in other words that the $\ell_1$-relaxations essentially lose the combinatorial structure of allowed sparsity patterns possibly encoded in $F$, 
 we focus, from now on, on the case $p>1$.
 
\paragraph{Lasso, group Lasso.}
$\Omega_p$ instantiates as the $\ell_1$, $\ell_p$ and $\ell_1/\ell_p$-norms for the simplest functions:
\BIT
\item If $F(A) = |A|$, then  $\Omega_p(\w) = \| \w \|_1$, since $\Omega_p^\ast(s) = \max_A \frac{ \| s_A\|_q}{|A|^{1/q}} = \| s \|_\infty$. It is interesting that the cardinality function is always relaxed to the $\ell_1$-norm for all $\ell_p$-relaxations, and is not an artifact of the traditional relaxation on an $\ell_\infty$-ball.
\item If $F(A) = 1_{\{A \neq \varnothing\}}$ , then   $\Omega_p(\w) = \| \w \|_p$, since
$\Omega_p^\ast(s) = \max_{A} \| s_A \|_q = \| s \|_q$.
\item If $F(A) \!=\! \sum_{j=1}^g 1_{\{ A \cap G_j  \neq \varnothing\}}$, for $(G_j)_{j \in \{1,\dots,g\}}$ a partition of $V$, then  $\Omega_p(\w) = \sum_{j=1}^g\| \w_{G_j} \|_p$ is the group Lasso or $\ell_1/\ell_p$-norm \citep{YuaLi06}. This result provides a principled derivation for the form of these norms, which did not exist in the literature. For groups which do not form a partition, this identity does in fact not hold in general for $p<\infty$, as we discuss in Section~\ref{sec:overlap}.
\EIT

\paragraph{Submodular functions and $p=\infty$.}
For a submodular function $F$ and in the $p=\infty$ case, the norm $\Omega^F_\infty$ that we derived actually coincides with the relaxation proposed by \citet{bach2010structured}, and as showed in that work, $\Omega^F_\infty(w)=f(|w|)$, where $f$ is a function associated with $F$ and called the \emph{\lova  extension} of $F$. We discuss the case of submodular functions in detail in Section~\ref{sec:submod}.

\subsection{Lower combinatorial envelope}
\label{sec:lce}
The fact that when $F$ is a submodular function, $\Omega^F_\infty$ is equal to the \lova extension $f$ on the positive orthant provides a guarantee on the tightness of the relaxation.
Indeed $f$ is called an ``extension" because $\forall A \subset 2^V,\: f(1_A)=F(A)$, so that $f$ can be seen to extend the function $F$ to $\RR^d$; as a consequence, $\Omega^F_\infty(1_A)=f(1_A)=F(A)$, which means that the relaxation is tight for all $w$ of the form $w=c\,1_A$, for any scalar constant $c \in \RR$ and any set $A \subset V$. If $F$ is not submodular, this property does not necessarily hold, thereby suggesting that the relaxation could be less tight in general.
To characterize to which extend this is true, we introduce a couple of new concepts.

Much of the properties of $\Omega_p$, for any $p>1$, are captured by the unit ball of $\Omega^*_\infty$ or its intersection with the positive orthant. 
In fact, as we will see in the sequel, the $\ell_\infty$ relaxation plays a particular role, to establish properties of the norm, to construct algorithms and for the statistical analysis, since it it reflects most directly the combinatorial structure of the function $F$.
 
We define the \emph{canonical polyhedron}\footnote{The reader familiar with submodular functions will recognize that the canonical polyhedron generalizes the submodular polyhedron usually defined for these functions.} associated to the combinatorial function as the polyhedron $\mathcal{P}_F$ defined by
$$\mathcal{P}_F = \{ s \in \rb^d, \ \forall A \subset V, \  s(A) \leq F(A) \}.$$

By construction, it is immediate that the unit ball of $\Omega^*_\infty$ is $\{s \in \rb^d \mid |s| \in \mathcal{P}_F\}.$

From this polyhedron, we construct a new set-function which restitutes the features of $F$ that are captured by $\mathcal{P}_F$:%
\begin{definition}[Lower combinatorial envelope]
Define the \emph{lower combinatorial envelope} (LCE) of $F$ as the set-function $\Fl$ defined by:
$$\Fl(A)= \max_{s \in \mathcal{P}_F} s(A).$$
\end{definition}
By construction, even when $F$ is not monotonic, $\Fl$ is always non-decreasing (because 
$\mathcal{P}_F  \subset \rb_+^d$).

One of the key properties of the lower combinatorial envelope is that, as shown in the next lemma,  $\Omega^F_\infty$ is an extension of $\Fl$
in the same way that the \lova extension is an extension of $F$ when $F$ is submodular.
\begin{lemma}(Extension property)
\label{lem:ext_lova}  $\displaystyle \Omega^F_\infty(1_A) \:=\:  \Fl(A) \rule{0mm}{1pc}$.
\end{lemma}
\begin{proof}
From the definition of $\Omega^F_\infty$, $\mathcal{P}_F$ and $F_-$, we get:
$\displaystyle \Omega^F_\infty(1_A)  \!=\! \max_{ {\Omega_\infty^F}^\ast(s)  \leq 1 } 1_A^\top \, s
 \!=\!  \max_{s \in \mathcal{P}_F} s^\top 1_A  \!=\!   \Fl(A) \rule{0mm}{1pc}$
\end{proof} 
Functions that are close to their LCE have in that sense a tighter relaxation than others.

\begin{figure}
\includegraphics[width=.95\tw]{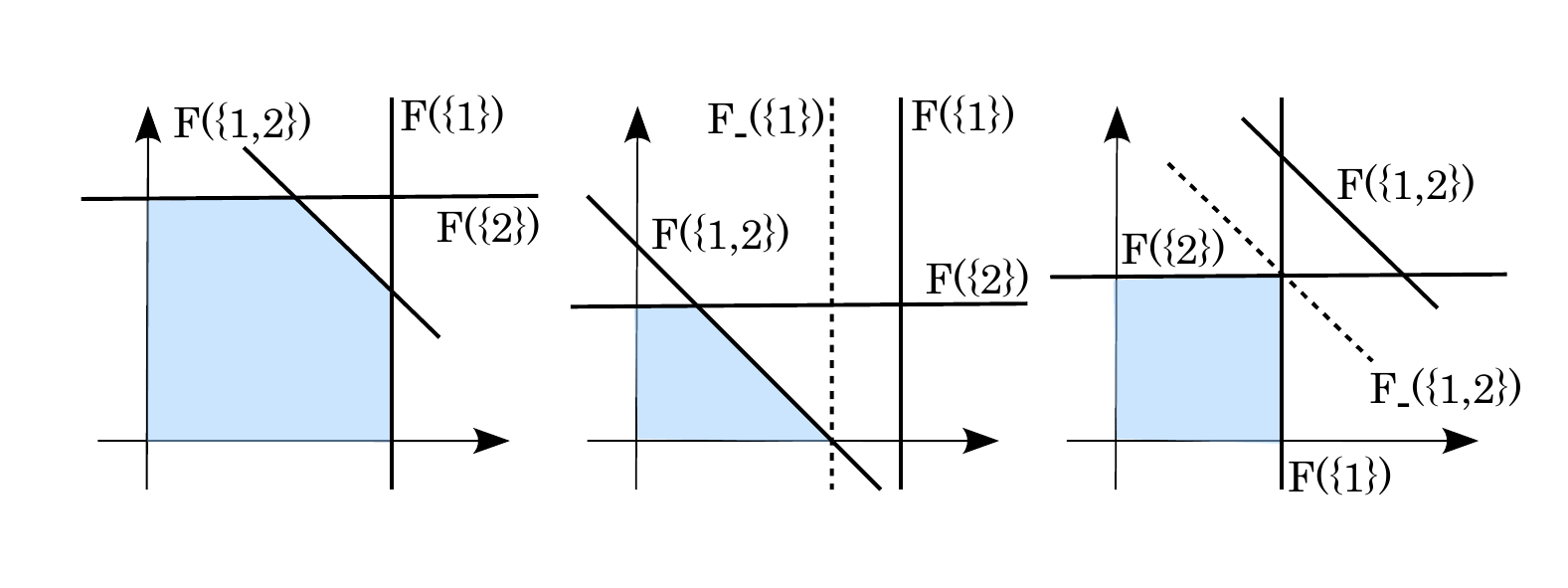}

\vspace*{-.5cm}

\caption{Intersection of the canonical polyhedron with the positive orthant for three different functions $F$. Full lines materialize the inequalities $s(A)\leq F(A)$ that define the polyhedron. Dashed line materialize the induced constraints $s(A) \leq \Fl(A)$ that results from all constraints $s(B)\leq F(B), \: B \in 2^V$. From left to right: (i) $\D_F=2^V$ and $\Fl=F=\Fu$; (ii) $\D_F= \{ \{2\},\{1,2\}\}$ and $\Fl(\{1\})<F(\{1\})$; (iii) $\D_F= \{ \{1\},\{2\}\}$ corresponding to a weighted $\ell_1$-norm.
}
\label{fig:submod_pol}
\end{figure}

A second important property is that a function $F$ and its LCE share the same canonical polyhedron. This will result as a immediate corollary from the following lemma:

\begin{lemma}
$\forall s \in \RR^V_+, \: \max_{A \subset V} \frac{s(A)}{F(A)}=\max_{A \subset V} \frac{s(A)}{\Fl(A)}$.
\end{lemma}
\begin{proof}
Given that for all $A$, $\Fl(A) \leqslant F(A)$, the left hand side is always smaller or equal to the right hand side.
 We now reason by contradiction. Assume that there exists $s$ such that  $\forall A \subset V, \: s(A) \leq \nu F(A)$
 but that there exists $B \subset V$,\: $s(B)> \nu F(B)$, then $s'=\frac{1}{\nu}s$ satisfies $\forall A \subset V, \: s'(A)\leq F(A)$. By definition of $\Fl$, the latter implies  that $\Fl(B)\geq s'(B)=\frac{1}{\nu}s(B)>\frac{1}{\nu}\cdot \nu F(B)$,
 where the last inequality results from the choice of this particular $B$. This would imply $\Fl(B)>F(B)$, but by definition of $\Fl$, we have $\Fl(A) \leq F(A)$ for all $A \subset V$.
\end{proof}

\begin{corollary}
$\mathcal{P}_F=\mathcal{P}_{\Fl}$.
\end{corollary}

But the sets $\{w \in \RR^d \mid |w| \in \mathcal{P}_F\}$ and  $\{w \in \RR^d \mid |w| \in \mathcal{P}_{\Fl}\}$ are respectively the unit balls of $\Omega^F_\infty$ and $\Omega^{\Fl}_\infty$. As a direct consequence, we have: 
\begin{lemma}
\label{sec:lce_same_norm}
 For all $p \geq 1, \quad \Omega^F_p=\Omega^{\Fl}_p$.
\end{lemma}
By construction, $\Fl$ is the largest function which lower bounds $F$, and has the same $\ell_p$-relaxation as $F$, hence the term of \emph{lower combinatorial envelope}.

Figure~\ref{fig:submod_pol} illustrates the fact that $F$ and $\Fl$ share the same canonical polyhedron and that the value of $\Fl(A)$ is determined by the values that $F$ takes on other sets. This figure also suggests that some constraints $\{\, s(A) \leq F(A) \,\}$ can never be active and could therefore be removed. This will be formalized in Section~\ref{sec:uce}. 

To illustrate the relevance of the concept of lower combinatorial envelope, we compute it for a specify combinatorial function, the range function, and show that it enables us to answer the question of whether the relaxation would be good in this case.

\begin{example}[Range function]
\label{ex:range}
Consider, on $V=[1,d]$, the range function $F: A \mapsto \max(A)- \min(A)+1$ where $\min(A)$ (resp. $\max(A)$) is the smallest (resp. largest) element in $A$.
A motivation to consider this function is that it induces the selection of supports that are exactly intervals. Since $F(\{i\})=1,\: i \in V$, then for all $s \in \mathcal{P}_F$, we have $s(A) \leq |A| \leq F(A)$. But this implies that $\D_F$ is the set of singletons and that $\Fl(A)=|A|$, so that $\Omega^F$ is the $\ell_1$-norm and is oblivious of the structure encoded in $F$. 
\end{example}

As we see from this example, the lower combinatorial envelope can be interpreted as the combinatorial function which the relaxation is actually able to capture.

\subsection{Upper combinatorial envelope} 
\label{sec:uce}

Let $F$ be a set-function and $\mathcal{P}_F$ its canonical polyhedron.
In this section, we follow an intuition conveyed by Figure~\ref{fig:submod_pol} and find a compact representation of $F$: the polyhedron $\mathcal{P}_F$ has in many cases a number of faces which much smaller than $2^d$. We formalize this in the next lemma.
\begin{lemma}
\label{lem:core}
There exists a minimal subset $\D_F$ of $2^V$ such that for $s \in \RR^d_+,$ 
$$s \in \mathcal{P}_F \Leftrightarrow (\forall A \in \D_F, \: s(A) \leq F(A)).$$
\end{lemma}
\begin{proof}
To prove the result, we define as in \citet[][Sec. 8.1]{obozinski2011group} the notion of \emph{redundant} sets: we say that a set $A$ is redundant for $F$ if 
$$\exists A_1, \ldots, A_k \in 2^V\backslash \{A\}, \quad (\forall i,\: s(A_i) \leq F(A_i)) \Rightarrow (s(A) \leq F(A)).$$
Consider the set $\D_F$ of all non redundant sets. 

We will show that, in fact, $A$ is redundant for $F$ if and only if 
$$\exists A_1, \ldots,A_k \in \D_F\backslash \{A\}, \quad (\forall i,\: s(A_i) \leq F(A_i)) \Rightarrow (s(A) \leq F(A)),$$ which proves the lemma.

Indeed, we can use a peeling argument to remove all redundant sets one by one and show recursively that the inequality constraint associated with a given redundant set is still implied by all the ones we have not removed yet. The procedure stops when we have reached the smallest set $\D_F$ of constraints implying all the other ones.
\end{proof}

We call $\D_F$ the \emph{core set} of $F$. It corresponds to the set of faces of dimension $d-1$ of $\mathcal{P}_F$.

This notion motivates the definition of a new set-function:
\begin{definition}(Upper combinatorial envelope)
We call \emph{upper combinatorial envelope} (UCE) the function $\Fu$ defined by $\Fu(A)=F(A)$ for $A \in \D_F$ and $\Fu(A)=\infty$ otherwise.
\end{definition}

As the reader might expect at this point, $\Fu$ provides a compact representation which captures all the information about $F$ that is preserved in the relaxation:

\begin{proposition}
\label{prop:lce_uce}
$F,\Fl$ and $\Fu$ all define  the same canonical polyhedron $\mathcal{P}_{\Fl}=\mathcal{P}_{F}=\mathcal{P}_{\Fu}$ and share the same core set $\D_F$. Moreover,
$\forall A \in \D_F,\: \Fl(A)=F(A)=\Fu(A).$
\end{proposition}
\begin{proof}
To show that $\Omega^{\Fu}_p=\Omega^{F}_p$ we just need to show $\mathcal{P}_{\Fu}=\mathcal{P}_F$. By the definition of $\Fu$ we have $\mathcal{P}_{\Fu}=\{s \in \RR^d \mid s(A) \leq F(A),\, A \in \D_F\}$ but the previous lemma precisely states that the last set is equal to $\mathcal{P}_F$.

We now argue that, for all $A \in \D_F, \: \Fl(A)=F(A)=\Fu(A)$. Indeed, the equality $F(A)=\Fu(A)$ holds by definition, and, for all $A \in \D_F$, we need to have $F(A)=\Fl(A)$
because $\Fl(A)=\max_{s \in \mathcal{P}_F} s(A)=\max_{s \in \mathcal{P}_{\Fu}} s(A)$ and if we had $\Fl(A)<F(A)$, this would imply that $A$ is redundant.
\end{proof}

Finally, the term ``upper combinatorial envelope" is motivated by the following lemma:

\begin{lemma}
\label{lem:uce_property}
$\Fu$ is the pointwise supremum of all the set-functions $H$ that are upper bounds on $F$ and such that $\mathcal{P}_H=\mathcal{P}_F$. 
\end{lemma}
\begin{proof}
We need to show that we have 
 $\Fu: A \mapsto \sup\{H(A) \mid H \in \mathcal{E}_F\}$ with  $$\mathcal{E}_F:=\{H: 2^V \rightarrow \overline{\RR}_+ \: \st \: H\geq F \: \text{and} \:\:  \mathcal{P}_H=\mathcal{P}_F\}.$$
But for any $H \in \mathcal{E}_F$, $\mathcal{P}_H=\mathcal{P}_F$ implies that $H_-=\Fl$ by definition of the lower combinatorial envelope. Moreover we have $\mathcal{D}_H\subset\mathcal{D}_F$ since if $A \notin \D_F$, then $A$ is redundant for $F$, i.e. there are $A_1, \ldots, A_k \in \D_F$ such that $\big (\forall i, \: s(A_i) \leq \Fl(A_i) \big ) \Rightarrow  (s(A) \leq F(A)),$ but $F(A) \leq H(A)$, and thus, since $\Fl(A_i)=H_-(A_i)$, this implies that $A$ is redundant for $H$, i.e., $A \notin \D_H$.
Now, assume there is $A \in \D_F \cap\D_H^c$. Since $A$ is redundant for $H$ then there are $A_1, \ldots, A_k$ in $\D_H\backslash \{A\}$ such that $\big (\forall i, \: s(A_i) \leq H(A_i) \big ) \Rightarrow  (s(A) \leq H(A)),$ but since $A_i \in\D_H \subset \D_F$ we have $H(A_i)=H_-(A_i)=\Fl(A_i)$ and since $A \in \D_F$ we also have $F(A)=\Fl(A)$ so either  $\big (\forall A' \in \D_H, \: s(A') \leq \Fl(A') \big ) \Rightarrow  (s(A) \leq \Fl(A)),$ so that $A$ is redundant, which is excluded, or since $H_-(A)=\max_{s \in \mathcal{P}_H} s (A)=\max_{s :\: s(A') \leq \Fl(A'), \: A' \in \D_H} s(A)$, we then have $H_-(A)>F_-(A)$, which is also impossible. So we necessarily have $\D_H=\D_F$.
To conclude the proof we just need to show that $\Fu \in \mathcal{E}_F$ and that $\Fu \geq H$ for all $H \in \mathcal{E}_F$;  this inequality is trivially satisfied for $A \notin \D_{\Fu}$, and since $\D_H=\D_{\Fu}$, for $A \in \D_{\Fu},$
 we have $\Fu(A)=\Fl(A)=H_-(A)=H(A)$.
 \end{proof}

The picture that emerges at this point from the results shown is rather simple: any combinatorial function $F$ defines a  polyhedron $\mathcal{P}_F$ whose faces of dimension $d-1$ are indexed by a set $\mathcal{D}_F \subset 2^V$ that we called the \emph{core set}. 
In symbolic notation: $\mathcal{P}_F=\{s \in \RR^d \mid s(A) \leq F(A), \: A \in \mathcal{D}_F\}$. All the combinatorial functions which are equal to $F$ on $\mathcal{D}_F$ and which otherwise take values that are larger than its lower combinatorial envelope $\Fl$, 
have the same $\ell_p$ tightest positively homogeneous convex relaxation $\Omega^F_p$, the smallest such function being $\Fl$ and the largest $\Fu$. Moreover $\Fl(A)=\Omega^F_\infty$(A), so that $\Omega^F_\infty$ is an extension of $\Fl$.
By construction, and even if $F$ is a non-decreasing function, $\Fl$ is non-decreasing, while $\Fu$ is obviously not a decreasing function, even though its restriction to $\D_F$ is.
It might therefore seem an odd set-function to consider; however if  $\D_F$ is a small set, since $\Omega^F_p=\Omega^{\Fu}_p$, and it provides a potentially much more compact representation of the norm, which we now relate to a norm previously introduced in the literature.

\section{Latent group Lasso, block-coding and set-cover penalties}
\label{sec:lgl}
The norm $\Omega_p$ is actually not a new norm. It was introduced from a different point of view by
\citet{jacob2009group} \citep[see also][]{obozinski2011group} as one of the possible generalizations of the group Lasso to the case where groups overlap.

To establish the connection, we now provide a more explicit form for $\Omega_p$, which is different from the definition via its dual norm which we have exploited so far.

We consider models that are parameterized by a vector $w \in \RR^V$ and associate to them latent variables that
are tuples of vectors of $\RR^V$ indexed by the power-set of $V$. Precisely, with the notation
$$\VV=\big \{\v=(\vA)_{A \subset V} \in \big ( \RR^{V} \big )^{2^V} \: \st \: \supp(\vA) \subset A \big \},$$ 
we define the 
norms $\Omega_p$  as
\vspace{-3mm}

 \BEQ
\label{eq:lgl}
\!\!\!\Omega_p(\w)  = \min_{\v \in \VV} \sum_{A \subset V} F(A)^\frac{1}{q} \, \|\vA\|_p \: \st \: \w=\sum_{A \subset V } \vA.
\EEQ

As suggested by notations and as first proved for $p=2$ by \citet{jacob2009group}, we have:
\begin{lemma} 
\label{lem:duality}
$\Omega_p$ and $\Omega_p^*$ are dual to each other.
\end{lemma}

An elementary proof of this result is provided in \citet{obozinski2011group}\footnote{The proof in \citet{obozinski2011group} addresses the $p=2$ case but generalizes immediately to other values of $p$.}. We propose a slightly more abstract proof of this result in appendix~\ref{sec:form_primal} using explicitly the fact that $\Omega_p$ is defined as an infimal convolution.

 We will refer to this norm $\Omega_p$ as the \emph{latent group Lasso} since it is defined by introducing latent variables $v^A$ that are themselves regularized instead of the original model parameters. We refer the reader to~\citet{obozinski2011group} for a detailed presentation of this norm, some of its properties and some support recovery results in terms of the support of the latent variables. In \citet{jacob2009group} the expansion \eqref{eq:lgl} did not involve all terms of the power-set but only a subcollection of sets $\G \subset 2^V$.
The notion of redundant set discussed in Section~\ref{sec:uce} was actually introduced by \citet[][Sec.~8.1]{obozinski2011group} and the set $\G$ could be viewed as the core set $\D_F$.
A result of \citet{obozinski2011group} for $p=2$ generalizes immediately to other $p$: the unit ball of $\Omega_p$ can be shown to be the convex hull of the sets $D_A=\{w \in \RR^d \mid \|w_A\|_p^p \leq F(A)^{-1/q}\}$. This is illustrated in Figure~\ref{fig:balls}.

The motivation of \citet{jacob2009group} was to find a convex regularization which would induce sparsity patterns that are unions of groups in $\G$ and explain the estimated vector $w$ as a combination of a small number of latent components, each supported on one group of $\G$. The motivation is very similar in \citet{huang2011learning} who consider an $\ell_0$-type penalty they call \emph{block coding}, where each support is penalized by the minimal sum of the \emph{coding complexities} of a certain number of elementary sets called ``blocks'' which \emph{cover} the support. In both cases the underlying combinatorial penalty is the \emph{minimal weighted set cover} defined for a set $B \subset V$ by: 
\BEAS
\label{eq:set-cover}
\Fc(B)& = & \min_{(\delta^A)_{A \subset V}} \:\sum_{A \subset V} F(A) \, \delta^A\qquad \st \qquad {\sum_{A \subset V} \: \delta^A 1_A \geq 1_B,\qquad \delta^A\in \{0,1\}, \: A \subset V}.
\EEAS
While the norm proposed by \citet{jacob2009group} can be viewed as a form of ``relaxation'' of the cover-set problem, a rigorous link between the $\ell_0$ and convex formulation is missing. We will make this statement rigorous through a new interpretation of the lower combinatorial envelope of $F$.

Indeed, assume w.l.o.g.~that $w \in \RR^d_+$. For $x,y \in \RR^V$, we write $x \geq y$ if $x_i \geq y_i$ for all $i \in V$. Then,
\BEAS
\Omega_\infty(w)& = & \min_{v \in \VV} \sum_{A \subset V} F(A) \|v^A\|_\infty  \qquad\st \qquad \sum_{A \subset V} v^A \geq w\\
& = & \min_{\delta^A \in \RR_+} \sum_{A \subset V} F(A) \, \delta^A \qquad \st \qquad \sum_{A \subset V} \delta^A 1_A \geq w,
\EEAS

since if $(v^A)_{A \subset V}$ is a solution so is $(\delta^A 1_A)_{A \subset V}$ with $\delta^A=\|v^A\|_\infty$.
We then have
\BEA
\label{eq:frac-set-cover}
\Fl(B) & = & \min_{(\delta^A)} \: \sum_{A \subset V} F(A) \, \delta^A,\qquad \st \qquad { \sum_{A \subset V} \delta^A 1_A \geq 1_B, \quad \delta^A \in [0,1], \: A \subset V},
\EEA
because constraining $\delta$ to the unit cube does not change the optimal solution, given that $1_B \leq 1$. But the optimization problem in
\eqref{eq:frac-set-cover} is exactly the \emph{fractional weighted set-cover problem} \citep{Lovasz1975Ratio}, a classical relaxation of the \emph{weighted cover set problem} in Eq.~(\ref{eq:set-cover}).

Combining Proposition~\ref{prop:relax} with the fact that $\Fl(A)$ is the fractional weighted set-cover, now yields:
\begin{theorem}
$\Omega_p(w)$ is the tightest convex relaxation of the function $w \mapsto \|w\|_p \, \Fc({\rm Supp}(w))^{1/q}$ where $\Fc({\rm Supp}(w))$ is the \emph{weighted set-cover} of the support of $w$.
\end{theorem}
\begin{proof}

We have $\Fl(A)\leq \Fc(A) \leq F(A)$ so that, since $\Fl$ is the lower combinatorial envelope of $F$, it is also the lower combinatorial envelope of $\Fc$, and therefore
$\Omega^{\Fl}_p=\Omega^{\Fc}_p=\Omega^F_p$.
\end{proof}
 This proves that the norm $\Omega_p^F$ proposed by 
\citet{jacob2009group} is indeed in a rigorous sense a relaxation of the block-coding or set-cover penalty. 

\begin{figure}
\begin{center}
\begin{tabular}{cc}
 \!\!\!\!\!\!\!\!\includegraphics[width=0.42\tw]{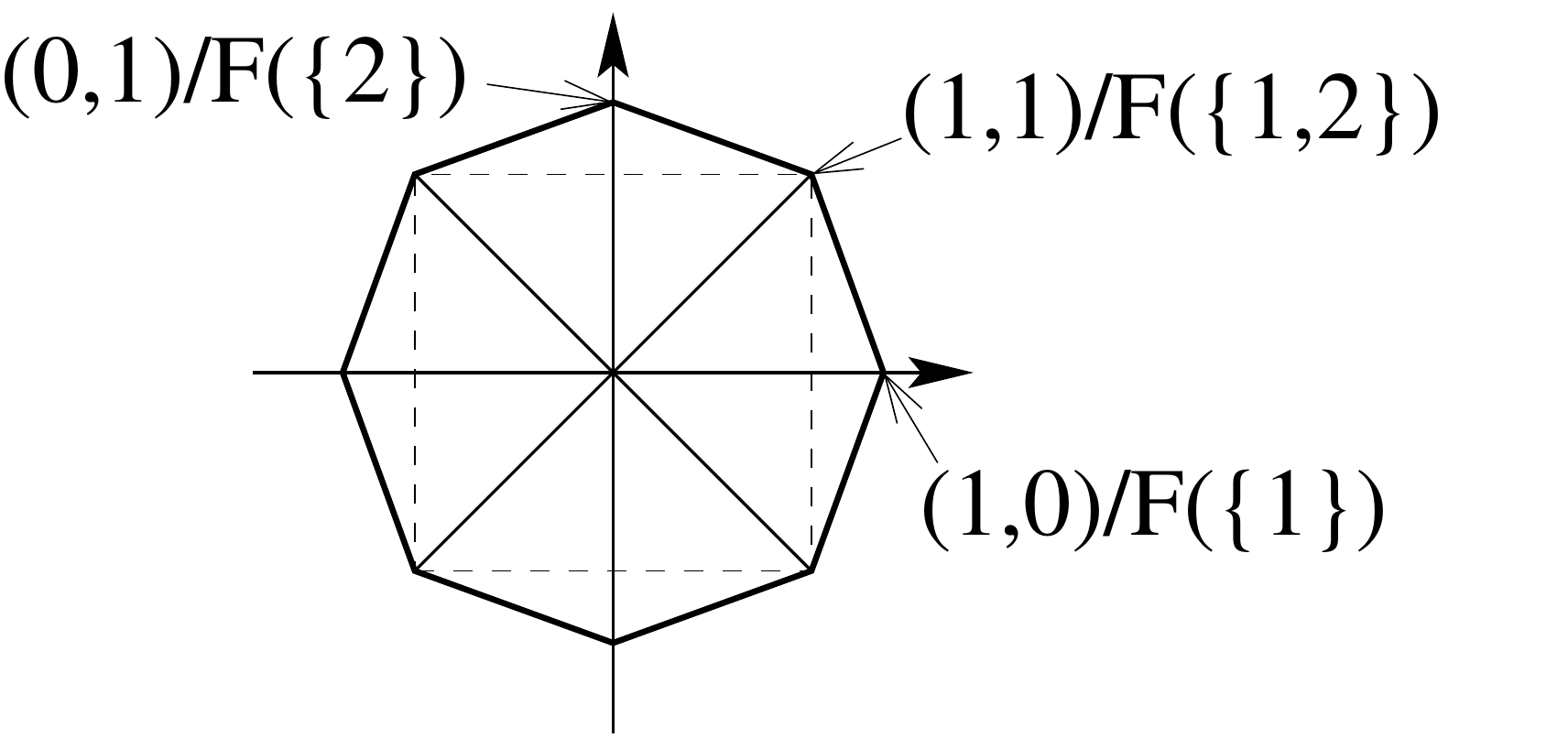}& \!\!\!\!\!\!\!\!\!\!\includegraphics[width=0.42\tw]{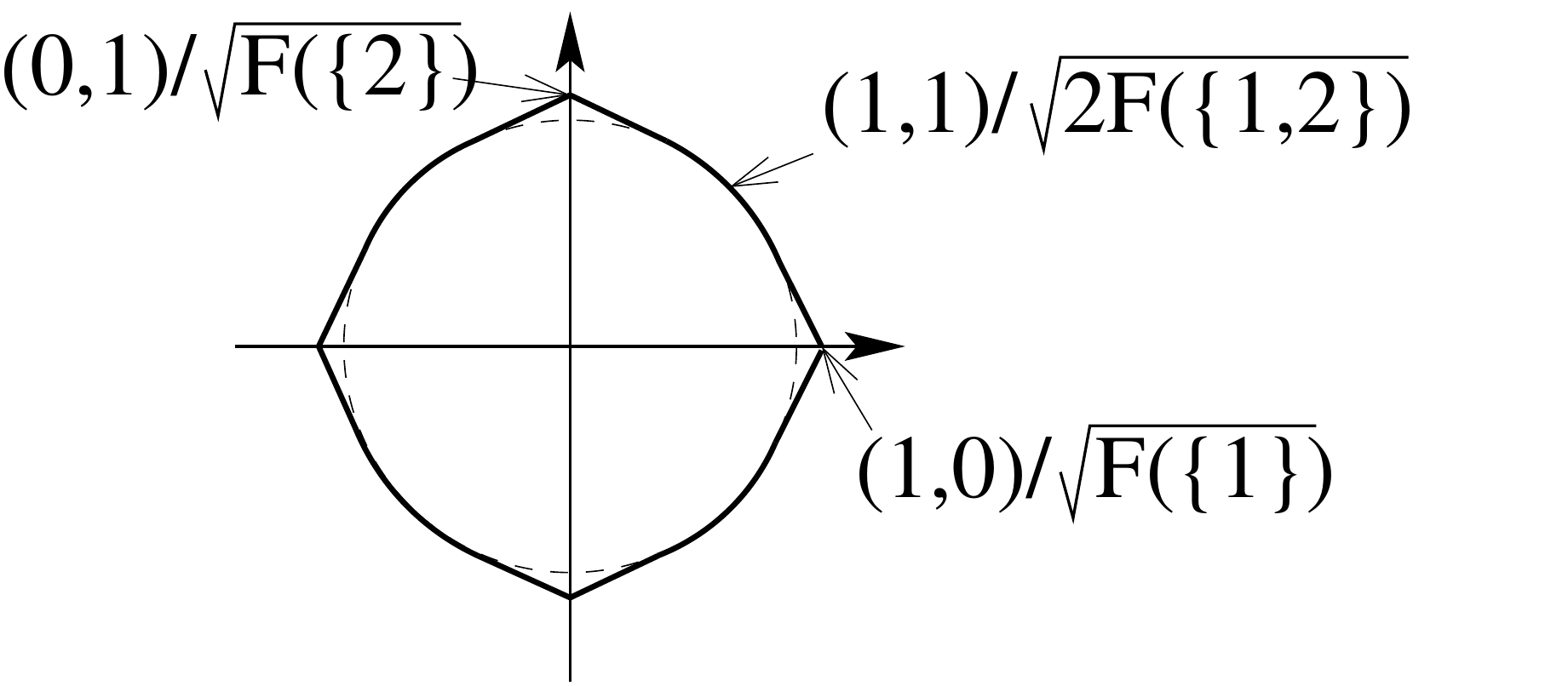} \\
\end{tabular}

\begin{tabular}{ccc}
\includegraphics[width=0.18\tw]{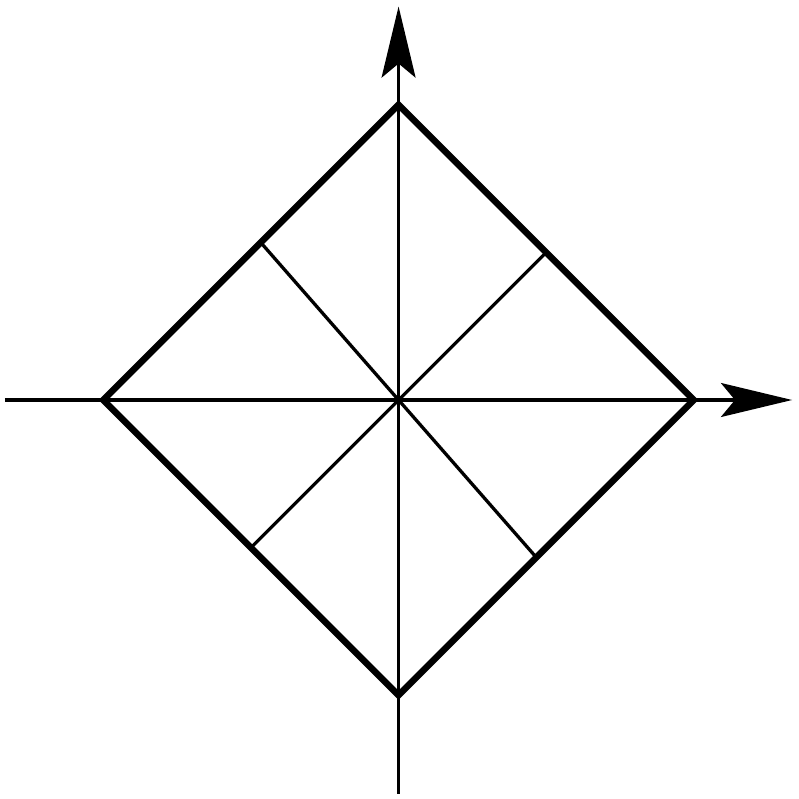} &
\includegraphics[width=0.18\tw]{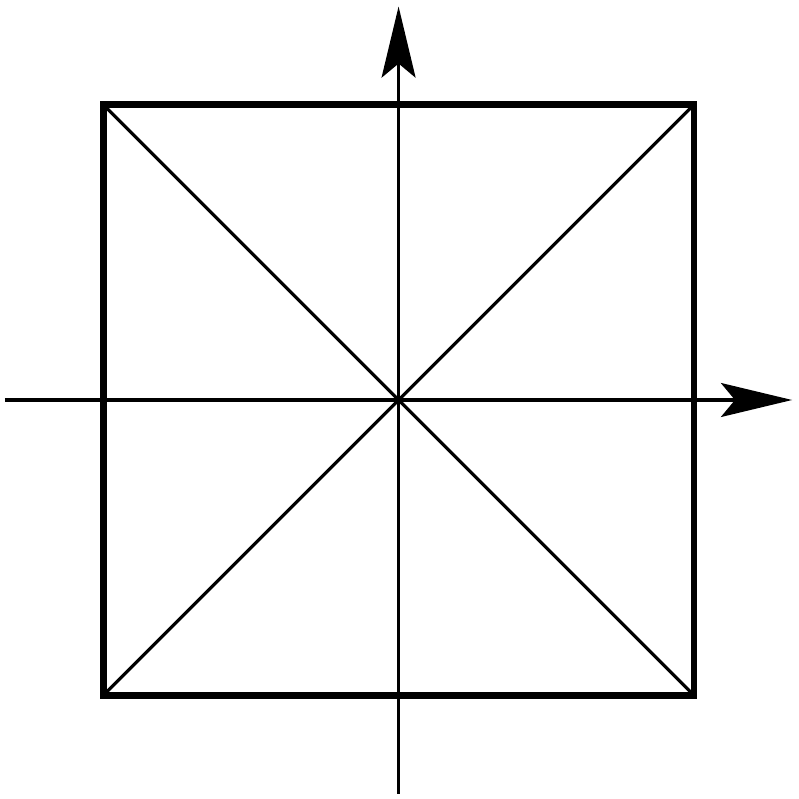} &
\includegraphics[width=0.18\tw]{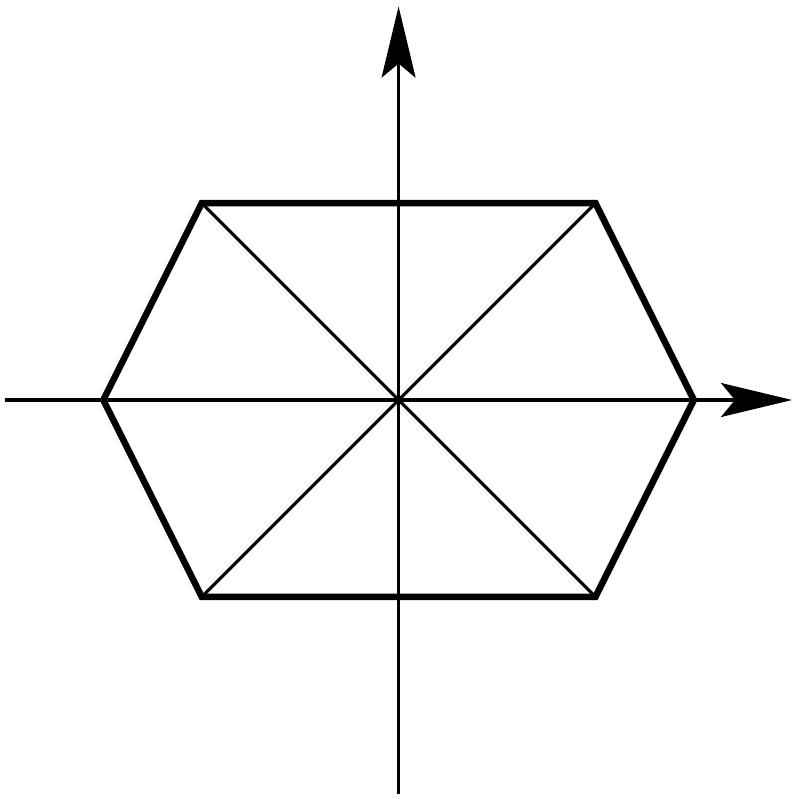} \\ 

\includegraphics[width=0.18\tw]{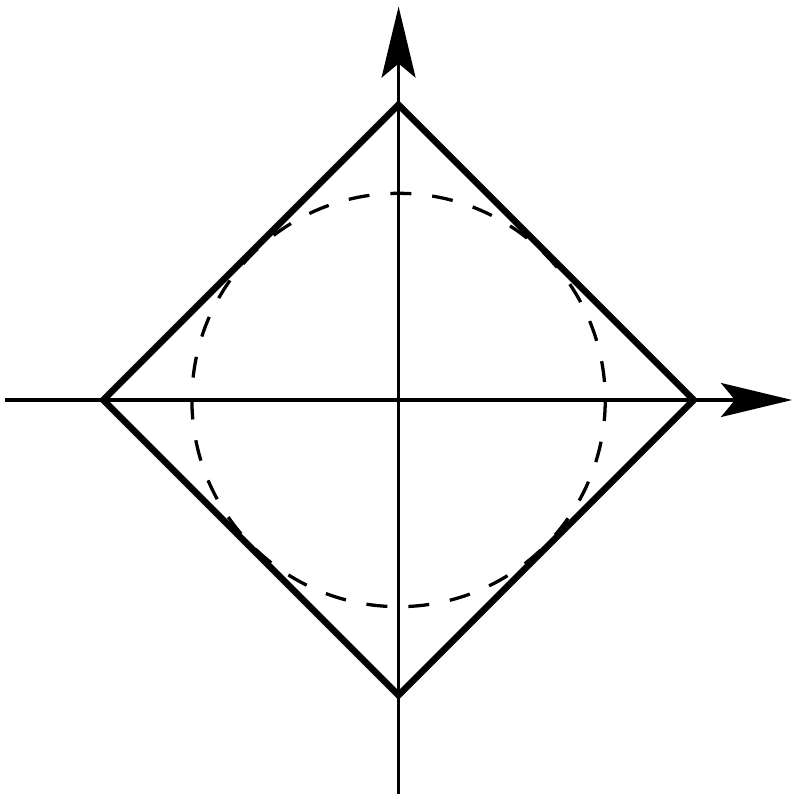} &
\includegraphics[width=0.18\tw]{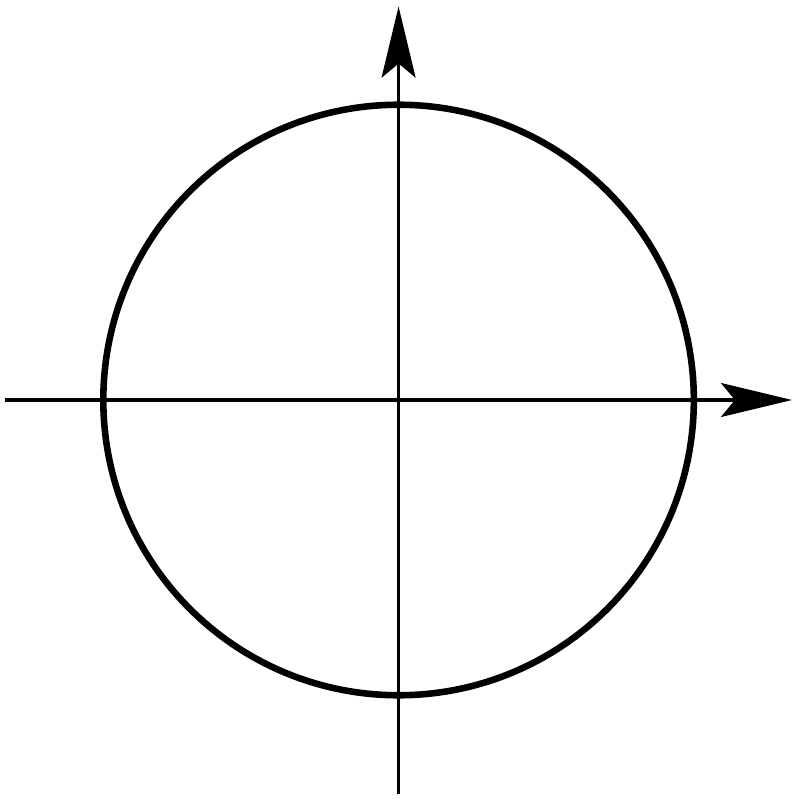} &
\includegraphics[width=0.18\tw]{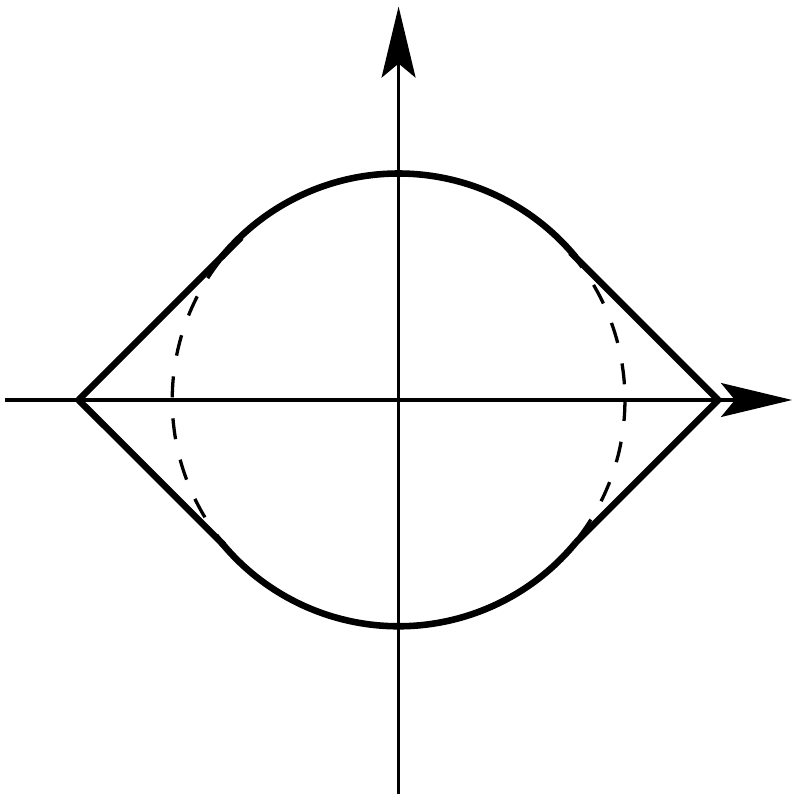}\\ 
\end{tabular}

\end{center}
\caption{Unit balls in $\RR^2$ for four combinatorial functions (actually all submodular) on two variables. 
Top left and middle row: $p=\infty$; top right and bottom row: $p=2$. Changing values of $F$ may make some of the extreme points disappear. All norms are hulls of a disk and points along the axes, whose size and position is determined by the values taken by $F$. On top row: $F(A)=\Fl(A) = |A|^{1/2}$ (all possible extreme points); and from left to right on the middle and bottom rows: $F(A) = |A|$ (leading to $\|\cdot\|_1$),
$F(A) =\Fl(A)= \min\{|A|,1\}$ (leading to $\|\cdot\|_p$), $F(A) =\Fl(A)=  \frac{1}{2} 1_{ \{A \cap \{2\} \neq \varnothing\} } +  1_{ \{A \neq \varnothing\}}$.
}
\label{fig:balls}
\end{figure}

\begin{example} 
\label{two_groups}
To illustrate the above results consider the block-coding scheme for subsets of $V=\{1,2,3\}$ with blocks consisting only of pairs, i.e., chosen from the collection $\D_0:=\big \{\{1,2\},\{2,3\},\{1,3\} \big\}$ with costs all equal to 1.
The following table lists the values of $F$, $\Fl$ and $\Fc$:
\def\cwz{6mm}
\def\cwo{8mm}
\def\cwt{12mm}

\begin{center}
\begin{tabular}{>{\centering} p{\cwo}|>{\centering}p{\cwz}|>{\centering}p{\cwo}|>{\centering}p{\cwo}|>{\centering}p{\cwo}|>{\centering}p{\cwt}|>{\centering}p{\cwt}|>{\centering}p{\cwt}| >{\centering} p{12mm}}
& $\varnothing$ & $\{1\}$ & $\{2\}$ & $\{3\}$ & $\!\!\{1,2\}$ & $\!\!\{2,3\}$& $\!\!\{1,3\}$& $\!\!\{1,2,3\}$\tabularnewline
\hline
$F$ & $0$ & $\infty$ & $\infty$& $\infty$ & $1$ &  $1$ & $1$ & $\infty$\tabularnewline
$\Fc$ & $0$ & $1$ &  $1$ & $1$ & $1$ &  $1$ & $1$ & $2$\tabularnewline
$\Fl$ & $0$ & $1$ &  $1$ & $1$ & $1$ &  $1$ & $1$ & $3/2$\tabularnewline
\end{tabular}
\end{center}
Here, $F$ is equal to its UCE (except that $\Fu(\varnothing)=\infty$) and takes therefore non trivial
values only on the core set $\D_F=\D_0$.
All non-empty sets except $V$ can be covered by exactly one set, which explains the cases where $\Fl$ and $\Fc$ take the value one. $\Fc(V)=2$ since $V$ is covered by any pair of blocks and a slight improvement is obtained if fractional covers is allowed since for $\delta_1=\delta_2=\delta_3=\frac{1}{2}$, we have $1_{V}=\delta_1\, 1_{\{2,3\}}+\delta_2\, 1_{\{3,1\}}+\delta_3\, 1_{\{1,2\}}$ and therefore $\Fl(V)=\delta_1+\delta_2+\delta_3=\frac{3}{2}$.  
\end{example}

The interpretation of the LCE as the value of a minimum fractional weighted set cover suggests a new interpretation of $\Fu$ (or equivalently of $\D_F$) as defining the smallest set of blocks ($\D_F$) and their costs, that induce a fractional set over problem with the same optimal value.

It is interesting to note (but probably not a coincidence) that it is \lova who introduced the concept of optimal fractional weighted set cover, while we just showed that the value of that cover is precisely $\Fl$, i.e.,
the combinatorial function which is extended by $\Omega^{\Fu}_\infty=\Omega^{\Fl}_\infty$ and which, if $\Fu$ is submodular is equal to the \lova extension.

The interpretation of $\Fl$ as the value of a minimum fractional weighted cover set problem allows us also to show a result which is dual to the property of LCEs, and which we now present.

\subsection{Largest convex positively homogeneous function with same combinatorial restriction}
By symmetry with the characterization of the \emph{lower combinatorial envelope} as the smallest combinatorial function that
has the same tightest convex and positively homogeneous (p.h.) relaxation as a given combinatorial function $F$, we can, given
a convex positively homogeneous function $g$, define the combinatorial function $F: A \mapsto g(1_A)$, which by construction, is the combinatorial function which $g$ \emph{extends} (in the sense of \lova) to $\RR^d_+$, and ask if there exists a largest convex and p.h.\ function $g^{+}$ among all such functions. It turns out that this problem is well-posed if the question is restricted to functions that are also coordinate-wise non-decreasing. Perhaps not surprisingly, it is then the case that the largest convex p.h.\ function extending the same induced combinatorial function is precisely $\Omega^F_\infty$, as we show in the next lemma.

\begin{lemma}(Largest convex positively homogeneous extension)
\label{lch_ext}
Let $g$ be a convex, p.h.\ and coordinate-wise non-decreasing function defined on $\RR^d_+$.
Define $F$ as $F: A \mapsto g(1_A)$ and denote by $\Fl$ its lower combinatorial envelope.\\
Then $F=\Fl$ and $\forall w \in \RR^d$, $g(|w|)\leq\Omega^F_\infty(w)$.
\end{lemma}
\begin{proof}
From Equation (\ref{eq:frac-set-cover}), we know that $\Fl$ can be written as the value of a minimal weighted fractional set-cover. But if $1_B \leq \sum_{A \subset V} \delta^A 1_A$, we have
$$\sum_{A \subset V} \delta^A g(1_A) \geq g \big ({ \textstyle \sum_{A \subset V} \delta^A }\big ) \geq g(1_B),$$
where the first inequality results from the convexity and homogeneity of $g$, and the second from the assumption that it is coordinate-wise non-decreasing. As a consequence, injecting the above inequality in (\ref{eq:frac-set-cover}), we have $\Fl(B) \geq F(B)$. But since, we always have $\Fl \leq F$, this proves the equality.

For the second statement, using the coordinate-wise monotonicity of $g$ and its homogeneity, we have $g(|w|) \leq \|w\|_\infty g(1_{\supp(w)})= \|w\|_\infty F(\supp(w))$. Then, taking the convex envelope of functions on both sides of the inequality
we get $g(|\cdot|)^{**} \leq \big( \|\cdot\|_\infty F(\supp(\cdot))\big )^{**}=\Omega^F_\infty$, where $(\cdot)^*$ denotes the Fenchel-Legendre transform.
\end{proof}

\section{Examples}
\label{sec:examples}
subsection{Overlap count functions, their relaxations and the $\ell_1/\ell_p$-norms.}
\label{sec:overlap}
\def\Fi{F_{\cap}}
\def\Flgl{F_{\cup}}
A natural family of set functions to consider are the functions that, given a collection of sets $\G \subset 2^V$ are defined as the number of these sets that are intersected by the support: 
\BEA
\label{eq:overlap_counting}
\Fi(A)=\sum_{B \in \G} d_B 1_{\{A\cap G \neq \varnothing\}}.
\EEA
Since $A \mapsto 1_{\{A\cap G \neq \varnothing\}}$ is clearly submodular and since submodular functions form a positive cone, all these functions are submodular, which implies that $\Omega_p^{\Fi}$ is a tight relaxation of $\Fi$.

\paragraph{Overlap count functions \emph{vs} set-covers.}

As mentioned in Section~\ref{sec:special}, if $\G$ is a partition, the norm $\Omega_p^{\Fi}$ is the $\ell_1/\ell_p$-norm; in this special case, $\Fi$ is actually the value of the minimal (integer-valued) weighted set-cover associated with the sets in $\G$ and the weights $d_G$.

However, it should be noted that, in general, the value of these functions is quite different from the value of a minimal weighted set-cover.
It has rather the flavor of some sort of ``maximal weighted set-cover" in the sense that any
set that has a non-empty intersection in the support would be included in the cover.
We call them overlap count functions.

\paragraph{$\ell_p$ relaxations of $\Fi$ \emph{vs} $\ell_1/\ell_p$-norms.}

In the case where $p=\infty$, \citet{bach2010structured} showed that even when groups overlap we have $\Omega_\infty(w)=\sum_{B \in \G} d_B \|w_G\|_\infty$, since the \lova extension of a sum of submodular functions is just the sum of the \lova extensions of the terms in the sum. 

The situation is more subtle when $p<\infty$: in that case, and perhaps surprisingly, $\Omega_p^{\Fi}$ is not the \emph{weighted $\ell_1/\ell_p$ norm with overlap} \citep{jenatton2011structured}, also referred to as the \emph{overlapping group Lasso} (which should clearly be distinguished from the \emph{latent group Lasso}) and which is the norm defined by $w \mapsto \sum_{B \in \G} d'_B \|w_G\|_p$. The norm $\Omega_p^{\Fi}$ does not have a simple closed form in general. In terms of sparsity patterns induced however, $\Omega_p^{\Fi}$  behaves like $\Omega_\infty^{\Fi}$, and as a result the sparsity patterns allowed by $\Omega_p^{\Fi}$ are the same as those allowed by the corresponding \emph{weighted $\ell_1/\ell_p$ norm with overlap}.

\paragraph{$\ell_p$-relaxation of $\Fi$ \emph{vs} latent group Lasso based on $\G$.}
It should be clear as well that $\Omega_p^{\Fi}$ is not itself the \emph{latent group Lasso} associated with the collection $\G$ and the weights $d_G$ in the sense of~\citet{jacob2009group}. Indeed, the latter corresponds to the function $\Flgl: A \mapsto 1_{\{A \neq \varnothing\}}+\iota_{\{A \in \G\}}$, or to its LCE which is the minimal value of the fractional weighted set cover associated with $\G$. Clearly, $\Flgl$ is in general strictly smaller than $\Fi$ and since the relaxation of the latter is tight, it cannot be equal to the relaxation of the former, if the combinatorial functions are themselves different. Obviously, the function $\Omega_p^{\Fi}$
is still as shown in this paper, another latent group Lasso corresponding to a fractional weighted set cover and involving a larger number of sets that the ones in $\G$ (possibly all of $2^V$).
This last statement leads us to what might appear to be a paradox, which we discuss next.

\paragraph{Supports \emph{stable by intersection} vs \emph{formed as unions}.}

\citet{jenatton2011structured} have shown that the family of norms they considered induces possible supports which form a family that is \emph{stable by intersection}, in the sense that the intersection of any two possible support is also a possible support. But since as mentioned above they have the same support as the norms $\Omega_{p}^{\Fi}$, for $1< p\leq \infty$,
which are latent group Lasso norms, and since \citet{jacob2009group} have discussed the fact that the supports induced by any norm $\Omega_p$ are formed by \emph{unions} of elements of the core set $\D$, is might appear paradoxical that the allowed support can be described at the same time as intersections and as unions. There is in fact not contradiction because in general the set of supports that are induced by the latent group Lasso are in fact not \emph{stable by union} in the sense that some unions are actually ``unstable" and will thus not be selected.

\paragraph{Three different norms.} To conclude, we must, given a set of groups $\G$ and a collection of weights $(d_G)_{G \in \G}$, distinguish three norms that can be defined from it, the weighted $\ell_1/\ell_p$-norm with overlap, the norm $\Omega_p^{\Fi}$ obtained as the $\ell_p$ relaxation of the submodular penalty $\Fi$, and finally, the norm $\Omega_p^{\F_\cup}$ obtained as the relaxation of the set-cover or block-coding penalty with the weights $d_G$.

Some of the advantages of using a tight relaxation still need to be assessed empirically and theoretically, but the possibility of using $\ell_p$-relaxation for $p<\infty$ removes the artifacts that were specific to the $\ell_\infty$ case. 

\subsection{Chains, trees and directed acyclic graphs.}
Instances of the three types of norms above are naturally relevant to induce sparsity pattern on structures such as chains, trees and directed acyclic graphs.

The weighted $\ell_1/\ell_p$-norm with overlap has been proposed to induce interval patterns on chains and rectangular or convex patterns on grids~\citep{jenatton2011structured}, for certain sparsity patterns on trees~\citep{jenatton2011proximal} and on directed acyclic graphs~\citep{mairal2011convex}.

One of the norm considered in~\citet{jenatton2011structured} provides a nice example of an overlap count function, which it is worth presenting.

\begin{example}[Modified range function]
\label{ex:submodular_range}
 A shown in Example~\ref{ex:range} in Section~\ref{sec:lce}, the natural range function on a sequence leads to a trivial LCE. Consider now the penalty with the form of Eq.~(\ref{eq:overlap_counting}) with $\G$ the set of groups defined as 
 
 $$\mathcal{G}=\big  \{[\![1,k]\!] \mid 1\leq k \leq p \big \} \cup \big  \{ [\![k,p]\!] \mid 1\leq k \leq p \big \}.$$
 
 A simple calculation shows that $\Fi(\varnothing)=0$ and that for $A \neq \varnothing$,
  $\Fi(A)=d-1+range(A)$. This function is submodular as a sum of submodular functions, and thus equal to it lower combinatorial envelope, which implies that the relaxation retains the structural a prior encoded by the combinatorial function itself. We will consider the $\ell_2$ relaxation of this submodular function in the experiments (see Section~\ref{sec:exp}) and compare it with the $\ell_1/\ell_2$-norm with overlap of~\citet{jenatton2011structured}.
\end{example}

In the case of trees and DAGs, a natural counting function to consider is the number of nodes which have at least one descendant in the support, i.e. functions of the form $\Fi: A \mapsto \sum_{i \in V} 1_{\{A \cap D_i \neq \varnothing\}}, $ where $D_i$ is the set containing node $i$ and all its descendants. It is related to the weighted $\ell_1/\ell_p-$norms which were considered in~\citet{jenatton2011proximal} ($p \in \{2, \infty\}$) for and~\citet{mairalpath} ($p=\infty$).
As discussed before, while these norms include $\Omega_\infty^{\Fi}$ if $p=\infty$, they otherwise do not correspond to the tightest relaxation, which it would be interesting to consider in future work.

Beyond the standard group Lasso and the exclusive group Lasso, there are very few instances of the norm $\Omega_2^F$ appearing in the literature. One such example is the wedge penalty considered in~\citet{Micchelli2011Regularizers}.
 
Latent group Lasso formulations are also of interest in these cases, and have not been yet been investigated much, with the exception of~\citet{mairalpath}, which considered the case of a parameter vector with coefficients indexed by a DAG and $\G$ the set of all paths in the graph.

There are clearly other combinatorial functions of interest than submodular functions and set-cover functions. We present an example of such functions in the next section.

\subsection{Exclusive Lasso}
\label{sec:exclusive}
The exclusive Lasso is a formulation proposed by \citet{zhou2010exclusive} which considers the case where a partition
$\mathcal{G}=\{G_1,\ldots,G_k\}$ of $V$ is given and the sparsity imposed is that $w$ should have at most one non-zero coefficient in each group $G_j$. The regularizer proposed by \citet{zhou2010exclusive} is the $\ell_p/\ell_1$-norm
defined\footnote{The Exclusive Lasso norm which is $\ell_p/\ell_1$ should not be confused with the group Lasso norm which is $\ell_1/\ell_p$.} by $\|w\|_{\ell_p/\ell_1}=(\sum_{G \in \G} \|w_{G}\|_1^p)^{1/p}$. Is this the tightest relaxation?

A natural combinatorial function corresponding to the desired constraint is the function $F(A)$ defined
by $F(\varnothing)=0$, $F(A)=1$ if $\max_{G \in \G} |A \cap G|=1$ and $F(A)=\infty$ otherwise. 

To characterize the corresponding $\Omega_p$ we can compute explicitly its dual norm $\Omega^*_p$:
\BEAS
\big (\Omega_p^*(w) \big )^q  &=&\max_{A \subset V, \, A \neq \varnothing} \frac{\|s_A\|_q^q}{F(A)}\\
&= & \max_{A \subset V} \: \|s_A\|_q^q \quad \st \quad |A \cap G| \leq 1,\, G \in \G\rule{0pc}{1.2pc}\\
&= & \max_{ i_j \in G_j, \, 1 \leq j \leq k} \sum_{j=1}^k |s_{i_j}|^q \:=\:   \sum_{j=1}^k \max_{i \in G_j}  |s_{i_j}|^q\: =\:  \sum_{j=1}^k \|s_{G_j}\|_\infty^q,\\
\EEAS
which shows that $\Omega_p^*$ is the $\ell_q/\ell_\infty$-norm or equivalently that $\Omega_p$ is the $\ell_p/\ell_1$-norm
and provides a theoretical justification for the choice of this norm: it is indeed the tightest relaxation!
It is interesting to compute the lower combinatorial extension of $F$ which is 
$\Fl(A)=\Omega^F_\infty(1_A)=\|1_A\|_{\ell_\infty/\ell_1}=\max_{G \in \G} |A \cap G|.$
This last function is also a natural combinatorial function to consider; by the previous result $\Fl$ has the same convex relaxation as F, but it would be however less obvious
to show directly that $\Omega^{\Fl}_p$ is the $\ell_p/\ell_1$ (see appendix~\ref{sec:exclusive_plus} for a direct proof which uses Lemma~\ref{lch_ext}).

\section{A variational form of the norm}
Several results on $\Omega_p$ rely on the fact that it can be related variationally to $\Omega_\infty$.
\begin{lemma}
$\Omega_p$ admits the two following variational formulations:
\BEAS
\Omega_p(w)
&=& \max_{\kappa \in \RR^d_+} \, \sum_{i \in V} \kappa_i^{1/q}|w_i| \quad \st \quad \forall A \subset V, \: \kappa(A) \leq F(A) \\  
 &=& \min_{\eta \in \RR^d_+}  \, \sum_{i \in V} \frac{1}{p}\frac{|w_i|^p}{\eta_i^{p-1}}+\frac{1}{q}\Omega_\infty(\eta).
\EEAS
\end{lemma}

\begin{proof}
Using Fenchel duality, we have:
\BEAS
\Omega_p(w)
&=&\max_{s \in \RR^d} \, s^\top w \quad \st \quad \Omega^*_p(w) \leq 1 \notag \\
&=& \max_{s \in \RR^d} \, s^\top w \quad \st \quad \forall A \subset V, \: \|s_A\|_q^q \leq F(A) \notag 
\mbox{ by definition of }  \Omega^*_p ,
\\
\label{eq:kappa-trick}
&=& \max_{\kappa \in \RR^d_+} \, \sum_{i \in V} \kappa_i^{1/q}|w_i| \quad \st \quad \forall A \subset V, \: \kappa(A) \leq F(A). 
\EEAS
But it is easy to verify that $\displaystyle \kappa_i^{1/q}|w_i| =\min_{\eta_i \in \RR_+} \frac{1}{p}\frac{|w_i|^p}{\eta_i^{p-1}}+\frac{1}{q} \eta_i \kappa_i$ with the minimum attained for $\eta_i=\frac{|w_i|}{\kappa_i^{1/p}}$.
We therefore get:
\BEAS
\Omega_p(w)&=& \max_{\kappa \in \RR^d_+} \min_{\eta \in \RR^d_+} \, \sum_{i \in V} \frac{1}{p}\frac{|w_i|^p}{\eta_i^{p-1}}+\frac{1}{q} \eta^\top \kappa \quad \st \quad \forall A \subset V, \: \kappa(A) \leq F(A) \notag \\
&=&\min_{\eta \in \RR^d_+}  \max_{\kappa \in \RR^d_+}  \, \sum_{i \in V} \frac{1}{p}\frac{|w_i|^p}{\eta_i^{p-1}}+\frac{1}{q} \eta^\top \kappa \quad \st \quad \forall A \subset V, \: \kappa(A) \leq F(A) \notag \\
\label{eq:eta-trick}
 &=& \min_{\eta \in \RR^d_+}  \, \sum_{i \in V} \frac{1}{p}\frac{|w_i|^p}{\eta_i^{p-1}}+\frac{1}{q}\Omega_\infty(\eta),
\EEAS
where we could exchange minimization and maximization since the function is convex-concave in $\eta$ and $\kappa$, and where we eliminated formally $\kappa$ by introducing the value of the dual norm 
$\Omega_\infty(\eta)= \max_{\kappa \in \mathcal{P}_F} \kappa^\top \eta$.
\end{proof}

Since $\Omega_\infty$ is convex, the last formulation is actually jointly convex in $(w,\eta)$ since $(x,z) \mapsto \frac{1}{p} \frac{\|x\|_p^p}{z^{p-1}} + \frac{1}{q} z$ is convex, as the perspective function of $t \mapsto t^p$ \citep[see][p.~89]{boyd}. 

It should be noted that the norms $\Omega_p$ therefore belong to the broad family of H-norms as defined\footnote{Note that H-norms are in these references defined for $p=2$ and that the variational formulation proposed here generalizes this to other values of $p \in (1, \infty)$} in~\citet[Sec.~1.4.2.]{bach2011optim} and studied by \citet{Micchelli2011Regularizers}.

The above result is particularly interesting if $F$ is submodular since $\Omega_\infty$ is then equal to the \lova extension of $F$ on the positive orthant \citep{bach2010structured}.
In this case in particular, it is possible, as we will see in the next section to propose efficient algorithms to compute $\Omega_p$ and $\Omega_p^*$, the associated proximal operators, and algorithms to solve learning problems regularized with $\Omega_p$
thanks to the above variational form.

For submodular functions, these variational forms are also the basis for the \emph{local decomposability} result of Section~\ref{sec:weak_dec} which is key to establish support recovery in Section~\ref{sec:theory}.

\section{The case of submodular penalties}

\label{sec:submod}
In this section, we focus on the case where the combinatorial function $F$ is submodular.

Specifically, we will consider a function $F$ defined on the power set $2^V$ of $V = \{1,\dots,d\}$, which is \emph{nondecreasing} and \emph{submodular}, meaning that it satisfies respectively
 \BEAS
   \forall A,B \subset V, & \qquad &  \quad A \subset B \Rightarrow F(A) \leqslant F(B),\\
  \EEAS
 Moreover, we assume that $F(\varnothing)=0$. 
 These set-functions are often referred to as \emph{polymatroid set-functions}~\citep{fujishige2005submodular,edmonds}. 
 Also, without loss of generality, we assume that $F$ is strictly positive on singletons, i.e., for all $k\in V$, $F(\{ k\})>0$. 
 Indeed, if $F(\{k\})=0$, then by submodularity and monotonicity, if $A \ni k$, $F(A) = F(A \backslash \{k\})$ and thus we can simply consider $V \backslash \{k\}$ instead of $V$.

Classical examples are the cardinality function and, given a partition of $V$ into $G_1 \cup \cdots \cup G_k=V$, the set-function $A\mapsto F(A)$ which is equal to 
the number of groups $G_1,\dots,G_k$ with non empty intersection with $A$, which, as mentioned in section \ref{sec:special} leads to the grouped $\ell_1$/$\ell_p$-norm.

With a slightly different perspective than the approach of this paper, \citet{bach2010structured} studied the special case of the norm $\Omega^F_p$ when $p=\infty$ and $F$ is submodular. As mentioned previously, he showed that in that case the norm $\Omega^F_\infty$ is the \lova extension of the submodular function $F$, which is a well studied mathematical object.

Before presenting results on $\ell_p$ relaxations of submodular penalties, we review a certain number of relevant properties and concepts  from submodular analysis. For more details, see, e.g.,~\citet{fujishige2005submodular}, and, for a review with proofs derived from classical convex analysis, see, e.g., \citet{bach2011learning}.
\subsection{Review of submodular function theory}
\label{sec:submod_review}

\paragraph{\lova extension.}  
 Given any set-function $F$, one can define its \emph{\lova extension} $f: \rb_+^d \to \rb$, as follows: given $w \in \rb_+^d$, we can order the components of $w$ in decreasing order $w_{j_1} \geqslant \dots \geqslant w_{j_p} \geqslant 0$, the value $f(w)$ is then defined as
 \BEA
\label{eq:lovasz1}
 f(w) &  = &  \sum_{k=1}^{p-1} ( x_{j_k} -x_{j_{k+1}}  ) F( \{ j_1,\dots,j_k\} ) + x_{j_p}  F( \{ j_1,\dots,j_p\} ) \\
\label{eq:lovasz2}
\textstyle
   & = &    \sum_{k=1}^{p} w_{j_k} [ F( \{ j_1,\dots,j_k\} ) - F( \{ j_1,\dots,j_{k-1}\} ) ] .
 \EEA
 
 The  \lova extension $f$ is always piecewise-linear, and when $F$ is submodular, it is also convex (see, e.g.,~\citet{fujishige2005submodular,bach2011learning}). Moreover,  for all $\delta \in \{0,1\}^d$, $f(\delta) = F(\supp(\delta))$ and $f$ is in that sense an extension of $F$ from  vectors in $\{0,1\}^d$ (which can be identified with indicator vectors of sets) to all vectors in $\rb_+^d$.
Moreover,  it turns out that minimizing $F$ over subsets, i.e., minimizing $f$ over $\{0,1\}^d$ is equivalent to minimizing $f$ over $[0,1]^d$~\citep{edmonds}.
 
\paragraph{Submodular polyhedron and norm} 

We denote by $\mathcal{P}$ the \emph{submodular polyhedron}~\citep{fujishige2005submodular}, defined as  the set of $s \in \rb_+^d$ 
such that for all $A \subset V$, $s(A) \leqslant F(A)$, i.e., 
$
\mathcal{P} = \{ s \in \rb_+^d, \ \forall A \subset V, \  s(A)\leqslant F(A) \},
$
where we use the notation $s(A) = \sum_{k \in A} s_k$. With our previous definitions, the submodular polyhedron is just the canonical polyhedron associated with a submodular function. One important result in submodular analysis is that, if $F$ is a nondecreasing submodular function, then we have a representation of $f$ as a maximum of linear functions~\citep{fujishige2005submodular,bach2011learning}, i.e.,  for all $w \in \rb_+^d$, 
\BEQ
\label{eq:poly}
f(w) = \max_{ s \in \mathcal{P}} \  w^\top s.
\EEQ
We recognize here that the \lova extension of a submodular function $F$ is directly related to the norm $\Omega_\infty^F$ in that $f(|w|)=\Omega_\infty^F(w)$ for all $w \in \rb^d$.
  
 \paragraph{Greedy algorithm} 
Instead of solving a linear program with $d+2^d$ constraints, a solution $s$ to \eqref{eq:poly} may be obtained by the following algorithm (a.k.a. ``greedy algorithm''):
order the components of $w$ in decreasing order $w_{j_1} \geqslant \dots \geqslant w_{j_d}$, and then take for all $k \in   V$,
$s_{j_k} = F( \{ j_1,\dots,j_k\} ) - F( \{ j_1,\dots,j_{k-1}\} ) .$
Moreover, if $w \in \rb^d$ has some negative components, then, to obtain a solution to $\max_{ s \in \mathcal{P}} \  w^\top s$,  we can take $s_{j_k}$ to be simply equal to zero for all $k$ such that $w_{j_k}$ is negative~\citep{edmonds}.

\paragraph{Contraction and restriction of a submodular function.}
Given a submodular function $F$ and a set $J$, two related functions, which are submodular as well, will play a crucial role both algorithmically and for the theoretical analysis of the norm.
Those are
 the \emph{restriction} of  $F$ to a set $J$, denoted $F_J$, and the \emph{contraction} of $F$ on $J$, denoted $F^J$. They are defined respectively as 
$$F_J: A \mapsto F(A \cap J) \qquad \text{and} \qquad F^J: A \mapsto F(A \cup J)-F(A).$$
Both $F_J$ and $F^J$ are submodular if $F$ is.

In particular the norms $\Omega_p^{F_J}: \RR^J \rightarrow \RR_+$ and $\Omega_p^{F^J}: \RR^{J^c} \rightarrow \RR_+$ associated respectively with $F_J$ and $F^J$ will be useful to ``decompose" $\Omega^F_p$ in the sequel.
We will denote these two norms by $\Omega_J$ and $\Omega^{J}$ for short. Note that their domains are not $\RR^d$ but the vectors with support in $J$ and $J^c$ respectively.

\paragraph{Stable sets.}
Another concept which will be key in this section is that of \emph{stable set}.
 A set $A$ is said \emph{stable} if it cannot be augmented without increasing $F$, i.e.,  if  for all sets $B \supset A$, $B \neq A \Rightarrow F(B) > F(A)$. If $F$ is strictly increasing (such as for the cardinality), then all sets are stable. 
  The set   of stable sets is closed by intersection. In the case $p=\infty$, \citet{bach2011learning} has shown that these stable sets were the only allowed sparsity patterns.
 
\paragraph{Separable sets.} 

A set $A$ is separable if we can find a partition of $A$ into $ A = B_1 \cup \cdots \cup B_k$ such that $F(A) = F(B_1)+\cdots+F(B_k)$. A set $A$ is inseparable if it is not separable. As shown in~\citet{edmonds}, the submodular polytope $\mathcal{P}$ has full dimension $d$ as soon as $F$ is strictly positive on all singletons, and its faces are exactly the sets $\{ s(A) = F(A) \}$ for stable \emph{and} inseparable sets~$A$. With the terminology that we introduced
in Section~\ref{sec:uce}, this means that the core set of $F$ is the set $\mathcal{D}_F$ of its stable and inseparable sets. In other words, we have $\mathcal{P} = \{ s \in \rb^d, \ \forall A \in \mathcal{D}_F, s(A) \leqslant F(A) \}$. The core set will clearly play a role when deriving concentration inequalities in \mysec{theory}. For the cardinality function, stable and inseparable sets are singletons. 
 
\subsection{Submodular function and lower combinatorial envelope}
A few comments are in order to confront submodularity to the previously introduced notions associated with cover-sets, and lower and upper combinatorial envelopes. We have showed that $\Fl(A)=\Omega_\infty(1_A)$.
But for a submodular function $\Omega_\infty(1_A)=f(1_A)=F(A)$ since $f$ is the \lova extension of $F$. This shows that a submodular function is its own lower combinatorial envelope. However the converse is not true: a lower combinatorial envelope is not submodular in general. Indeed, in example~\ref{two_groups}, we have $\Fl(\{1,2\})+\Fl(\{2,3\}) \ngeq \Fl(\{2\})+\Fl(\{1,2,3\})$.

The core set of a submodular function is the set $\mathcal{D}_F$ of its stable and inseparable sets, which implies that $F$ can be retrieved as the value of the minimal fractional weighted set cover the sets $A \in \mathcal{D}_F$ with weights $F(A)$.
 
\subsection{Optimization algorithms for the submodular case}
\label{sec:optimization}
In the context of sparsity and structured sparsity, \emph{proximal methods} have emerged as methods of choice to design efficient algorithm to minimize objectives of the form
$f(w)+\lambda \Omega(w),$ where $f$ is a smooth function with Lipschitz gradients and $\Omega$ is a proper convex function \citep{bach2011optim}. In a nutshell, their principle is to linearize $f$ at each iteration
and to solve the problem
$$\min_{w \in \RR^d} \nabla f(w_t)^\top(w-w_t)+\frac{L}{2}\|w-w_t\|^2+ \lambda \Omega(w),$$ 
for some constant $L$.
This problem is a special case of the so-called \emph{proximal problem}:
\BEQ
\label{eq:prox}
\min_{w \in \rb^d} \frac{1}{2} \| w - z \|_2^2 + \lambda \Omega_p(w).
\EEQ
The function mapping $z$ to the solution of the above problem is called \emph{proximal operator}. If this proximal operator can be computed efficiently, then proximal algorithm provide good rates of convergence
especially for strongly convex objectives.
We show in this section that the structure of submodular functions can be leveraged to compute efficiently $\Omega_p$, $\Omega_p^*$ and the proximal operator.

\subsubsection{Computation of $\Omega_p$ and $\Omega_p^*$. }
A simple approach to compute the norm is to maximize in $\kappa$ in the variational formulation~(\ref{eq:kappa-trick}).
This can be done efficiently using for example a \emph{conditional gradient} algorithm, given that maximizing a linear form
over the submodular polyhedron is done easily with the \emph{greedy algorithm} (see Section \ref{sec:submod_review}).

We will propose another algorithm to compute the norm based on the so-called \emph{decomposition algorithm}, which is a classical algorithm of the submodular analysis literature that makes it possible to minimize a separable convex function over the submodular polytope efficiently~\citep[see, e.g.,][Section 8.6]{bach2011learning}.

Since the dual norm is defined as $\Omega_p^*(s)=\max_{A \subset V, A \neq \varnothing} \frac{\|s_A\|_q}{F(A)^{1/q}}$, to compute it from $s$, we need to maximize efficiently over $A$, which can be done, for submodular functions, through a sequence of submodular function minimizations~\citep[see, e.g.,][Section 8.4]{bach2011learning}.

\subsubsection{Computation of the proximal operator}

Using \eq{kappa-trick}, we can reformulate problem~(\ref{eq:prox}) as
\BEAS
\min_{w \in \rb^d} \frac{1}{2} \| w - z \|_2^2 + \lambda \Omega_p(w)
& = & 
\min_{w \in \rb^d}  \max_{\kappa \in \rb^d_+ \cap \mathcal{P} } \frac{1}{2} \| w - z \|_2^2 + \lambda 
\sum_{i \in V} \kappa_i^{1/q} | w_i |
\\
& = & 
\max_{\kappa \in \rb^d_+ \cap \mathcal{P} } \sum_{i \in V}  \min_{w_i \in \rb }  \bigg\{ \frac{1}{2}( w_i - z_i)^2 + \lambda 
\kappa_i^{1/q} | w_i | \bigg\}
\\
& = & 
\max_{\kappa \in \rb^d_+ \cap \mathcal{P}} \sum_{i \in V}  \psi_i(\kappa_i),\\
\EEAS
with $\psi_i: \kappa_i \mapsto \min_{w_i \in \rb }  \bigg\{ \frac{1}{2}( w_i - z_i)^2 + \lambda 
\kappa_i^{1/q} | w_i | \bigg\}.$

Thus, solving the proximal problem is equivalent to  maximizing a concave separable function $\sum_i \psi_i(\kappa_i)$ over the submodular polytope. For a submodular function, this can be solved with a ``divide and conquer" strategy which takes the form of the so-called \emph{decomposition algorithm} involving a sequence of submodular function minimizations~\citep[see][]{groenevelt1991two,bach2011learning}. This yields an algorithm which finds a decomposition of the norm and applies recursively the proximal algorithm to the
two parts of the decomposition corresponding respectively to a \emph{restriction} and a \emph{contraction} of the submodular function. We explicit this algorithm as Algorithm 1 for the case $p=2$.

\begin{algorithm}[hbtp]
\caption{Computation $x=\text{Prox}_{\lambda \Omega_2^F}(z)$}
\begin{algorithmic}[1]
\REQUIRE $z \in \RR^d, \: \lambda>0$
\STATE Let $A=\{j \mid z_j \neq 0\}$
\IF{$A \neq V$}
\STATE Set $x_A=\text{Prox}_{\lambda \Omega_2^{F_A}}(z_A)$
\STATE Set $x_{A^c}=0$
\STATE \textbf{return} $x$ by concatenating $x_A$ and $x_{A^c}$
\ENDIF
\STATE Let $t \in \RR^d$ with $t_i=\frac{z_i^2}{\|z\|_2} F(V)$
\STATE Find $A$ minimizing the submodular function $F-t$
\IF {$A=V$}
\STATE \textbf{return}  $x=\big (\|z\|_2-\lambda\sqrt{F(V)} \big )_+ \frac{z}{\|z\|_2}$
\ENDIF
\STATE Let $x_A=\text{Prox}_{\lambda \Omega_2^{F_A}}(z_A)$
\STATE Let $x_{A^c}=\text{Prox}_{\lambda \Omega_2^{F^A}}(z_{A^c})$
\STATE \textbf{return}  $x$ by concatenating $x_A$ and $x_{A^c}$
\end{algorithmic}
\end{algorithm}

Applying this decomposition algorithm in the special case where $\lambda=0$ yields a decomposition
algorithm, namely Algorithm \ref{alg:norm_decomp}, to compute the norm itself (see appendix~\ref{sec:decomp_alg}).

\subsection{Weak and local decomposability of the norm for submodular functions.}   
\label{sec:weak_dec}

The work of~\citet{negahban2010unified} has shown that when a norm is \emph{decomposable with respect to a pair of subspaces} $A$ and $B$, meaning that for all $\alpha \in A$ and $\beta \in B^\bot$ we have $\Omega(\alpha+\beta)=\Omega(\alpha)+\Omega(\beta)$, a common proof scheme allows to show support recovery results and fast rates of convergence in prediction error. For the norms we are considering, this type of assumption would be too strong. Instead, we follow the analysis of~\citet{bach2010structured} which considered the case $p=\infty$ and which only requires some weaker form of decomposability. The decompositions involve $\OJ$ and $\OnJ$ which are respectively the norms associated with the \emph{restriction} and the \emph{contraction} of the submodular function $F$ to or on the set $J$.

Concretely, let 
$c=\frac{\tilde{m}}{M}$ with $M=\max_{k \in V} F(\{k\})$ and
$$\tilde{m}=\min_{A,k}\: F(A \cup \{k\} ) - F(A) \:\st \: F(A \cup \{k\} ) > F(A).$$ 
Then we have: 
\begin{proposition}(Weak and local decomposability)

\label{prop:inequality}
\textbf{Weak decomposability.} For any set $J$ and any $w \in \rb^d$, we have $$\Omega(\w) \geq \OJ(\w_J)+\OnJ(\w_{J^c}).$$ 

\textbf{Local decomposability.} Let $K=\supp(w)$ and $J$ the smallest stable set containing $K$, if $\| \w_{J^c} \|_p \leq c^{1/p} \min_{i \in K} |w_i|$, then $$\Omega(\w) = \OJ(\w_J)+\OnJ(\w_{J^c}).$$
\end{proposition}
Note that when $p=\infty$, if $J=K$, the condition becomes  $\min_{i \in J} |w_i| \geqslant \max_{i \in J^c} |w_i|$, and we recover exactly the corresponding result from \citet{bach2010structured}.

This proposition shows that  a sort of reverse triangular inequality involving the norms $\Omega, \OJ$ and $\OnJ$ always holds and that if there is a sufficiently large positive gap between the values of $w$ on $J$ and on its complement then $\Omega$ can be written as a separable function on $J$ and $J^c$. 

\subsection{Theoretical analysis for submodular functions}
\label{sec:theory}
In this section, we consider a fixed design matrix $X \in \rb^{n \times p}$  and $y \in \rb^n$ a vector of random responses. Given $\lambda >0$, we define
$\hat{w}$ as a minimizer of the regularized least-squares cost:
\BEQ
\label{eq:objective}
\textstyle
\min_{w \in \rb^d} \textstyle \frac{1}{2n} \| y - X w\|_2^2 + \lambda \Omega(w).
\EEQ
We study the sparsity-inducing properties of solutions of (\ref{eq:objective}), i.e., we determine  which patterns are allowed and  which sufficient conditions lead to correct estimation.

We assume that the linear model is well-specified and extend results from~\citet{Zhaoyu} for sufficient support recovery conditions and from~\citet{negahban2010unified} for estimation consistency, which were already derived by~\citet{bach2010structured} for $p=\infty$.  
The following propositions allow us to retrieve and extend well-known results for the $\ell_1$-norm. 
 
Denote by $\rho$ the following constant:
$$
\rho = \min_{ A \subset B, F(B)>F(A)} \frac{ F(B) - F(A) }{F(B \backslash A)} \in (0,1].
$$
The following proposition
extends results based on support recovery conditions~\citep{Zhaoyu}:
\begin{proposition}[\textbf{Support recovery}]
\label{prop:support}
Assume that $y = Xw^\ast + \sigma \varepsilon$, where $\varepsilon$ is a standard multivariate normal vector. Let $Q = \frac{1}{n} X^\top X \in \rb^{d \times d}$.  Denote by $J$ the smallest stable set containing the  support $\supp(w^\ast)$ of $w^\ast$. Define $\nu = \min_{ j, w^\ast_j \neq 0 } | w^\ast_j | >0$ and assume $\kappa = \lambda_{\min} (Q_{JJ}) > 0$. 

If the following \textbf{generalized Irrepresentability Condition} holds:
$$\exists \eta >0, \qquad (\Omega^J)^\ast \Big (  \big ( \Omega_J(   Q_{JJ}^{-1} Q_{Jj} ) \big )_{j \in J^c} \Big ) \leqslant 1 - \eta,$$
 then, if $\lambda \leqslant \frac{\kappa \nu }{2 |J|^{1/p} F(J)^{1-1/p}} $, the minimizer $\hat{w}$ is unique and has support equal to $J$, with   probability larger than 
 $
1-3 \, \mathbb{P} \big( \Omega^\ast(z) > \frac{ \lambda \eta \rho   \sqrt{n} }{ 2 \sigma   } \big)
 $,
 where $z$ is a multivariate normal with covariance matrix $Q$.
 
 \end{proposition}
 
In terms of prediction error the next proposition 
extends results based on restricted eigenvalue conditions~\citep[see, e.g.\ ][]{negahban2010unified}. 
\begin{proposition}[\textbf{Consistency}]
\label{prop:high-dim}
Assume that $y = Xw^\ast + \sigma \varepsilon$, where $\varepsilon$ is a standard multivariate normal vector. Let $Q = \frac{1}{n} X^\top X \in \rb^{d \times d}$. Denote by $J$ the smallest stable set containing the  support $\supp(w^\ast)$ of $w^\ast$. 

If the following \textbf{$\Omega_J$-Restricted Eigenvalue condition} holds:

$$\forall \Delta \in \RR^d, \qquad \big (\:\Omega^J(\Delta_{J^c}) \leqslant 3 \Omega_J(\Delta_J) \: \big ) \quad \Rightarrow\quad \big ( \: \Delta^\top Q \Delta \geqslant \kappa \, \Omega_J(\Delta_J)^2 \: \big ),$$ then we have $$  \Omega(\hat{w} - w^\ast) \leqslant  \frac{24  ^2 \lambda}{\kappa \rho ^2}  \qquad \mbox{ and } \qquad  \frac{1}{n} \| X \hat{w} - X w^\ast\|_2^2 \leqslant 
  \frac{36   \lambda^2}{\kappa \rho ^2},$$
 with probability larger than 
$1 - \, \mathbb{P} \big( \Omega^\ast(z) > \frac{ \lambda \rho  \sqrt{n}  }{ 2 \sigma  } \big) $
where $z$ is a multivariate normal with covariance matrix $Q$.
   \end{proposition}
   
   The concentration of the values of $\Omega^\ast(z)$ for $z$ is a multivariate normal with covariance matrix $Q$ can be controlled via the following result.
   
 \begin{proposition} 
  \label{prop:proba}
 Let $z$ be a normal variable with covariance matrix $Q$ that has  unit diagonal. Let $\mathcal{D}_F$ be the set of stable inseparable sets. Then
\BEQ \mathbb{P} \bigg(
\Omega^\ast(z) \geqslant 4 \sqrt{ q \log( 2 |\mathcal{D}_F|)} \max_{A \in \mathcal{D}_F}  \frac{ |A|^{1/q}}{F(A)^{1/q}} 
+ u \max_{A \in \mathcal{D}_F} \frac{|A|^{(1/q-1/2)_+}}{   F(A)^{1/q} }
\bigg) \leqslant e^{-u^2/2}.
\EEQ
   \end{proposition}
\section{Experiments}

\label{sec:exp}

\subsection{Setting}

To illustrate the results presented in this paper we consider the problem of estimating the support of a parameter vector $w \in \RR^d$, when its support is assumed either 
\BIT
\item[(i)] to form an \emph{interval} in $[\![1,d\,]\!]$ or 
\item[(ii)] to form a \emph{rectangle} $[\![k_{\min}, k_{\max}]\!] \times [\![k'_{\min}, k'_{\max}]\!] \subset [\![1,d_1]\!] \times [\![1,d_2]\!]$, with $d=d_1d_2$.
\EIT

These two settings were considered in \citet{jenatton2011structured}. These authors showed  that, for both types of supports, it was possible to construct an $\ell_1/\ell_2$-norm with overlap based on a well-chosen collection of overlapping groups, so that the obtained estimators almost surely have a support of the correct form. Specifically, it was shown in \citet{jenatton2011structured} that norms of the form $w \mapsto \sum_{B \in \mathcal{G}} \|w_B\|_2$ induce sparsity patterns that are exactly intervals of $V=\{1, \ldots,p\}$ if 
$$\mathcal{G}=\big  \{[1,k] \mid 1\leq k \leq p \big \} \cup \big  \{ [k,p] \mid 1\leq k \leq p \big \},$$ and induce rectangular supports on $V=V_1 \times V_2$ with $V_1:=\{1, \ldots,p_1\}$ and $V_2:=\{1, \ldots,p_2\}$ if 
\begin{eqnarray*} \mathcal{G}&=&\big  \{ [\![1, k]\!] \times V_2 \mid 1\leq k \leq p_1\big \} \cup \big  \{ [\![k, p_1]\!]\times V_2 \mid 1\leq k \leq p_1 \big  \}\\ 
& & \cup  \big \{ V_1 \times [\![1, k]\!] \mid 1\leq k \leq p_2 \big \}  \cup  \big  \{ V_1 \times [\![k,p_2]\!] \} \mid 1\leq k \leq p_2 \big \}.
\end{eqnarray*}

These sets of groups are illustrated on Figure~\ref{fig:groups}, and, for the first case, the set $\G$ has already discussed in Example~\ref{ex:submodular_range} to define a modified range function which is submodular.

\begin{figure}
\includegraphics[width=.45\tw]{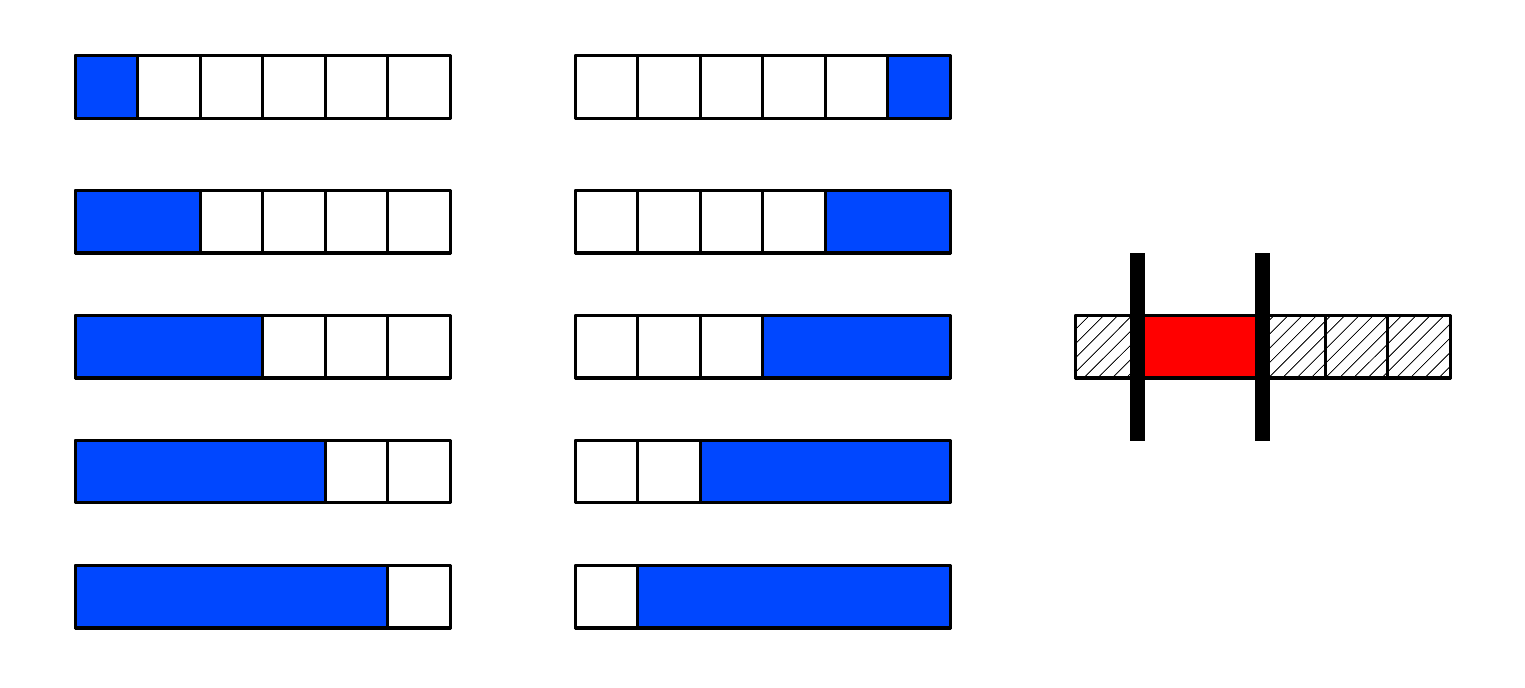} \hfill
\includegraphics[width=.45\tw]{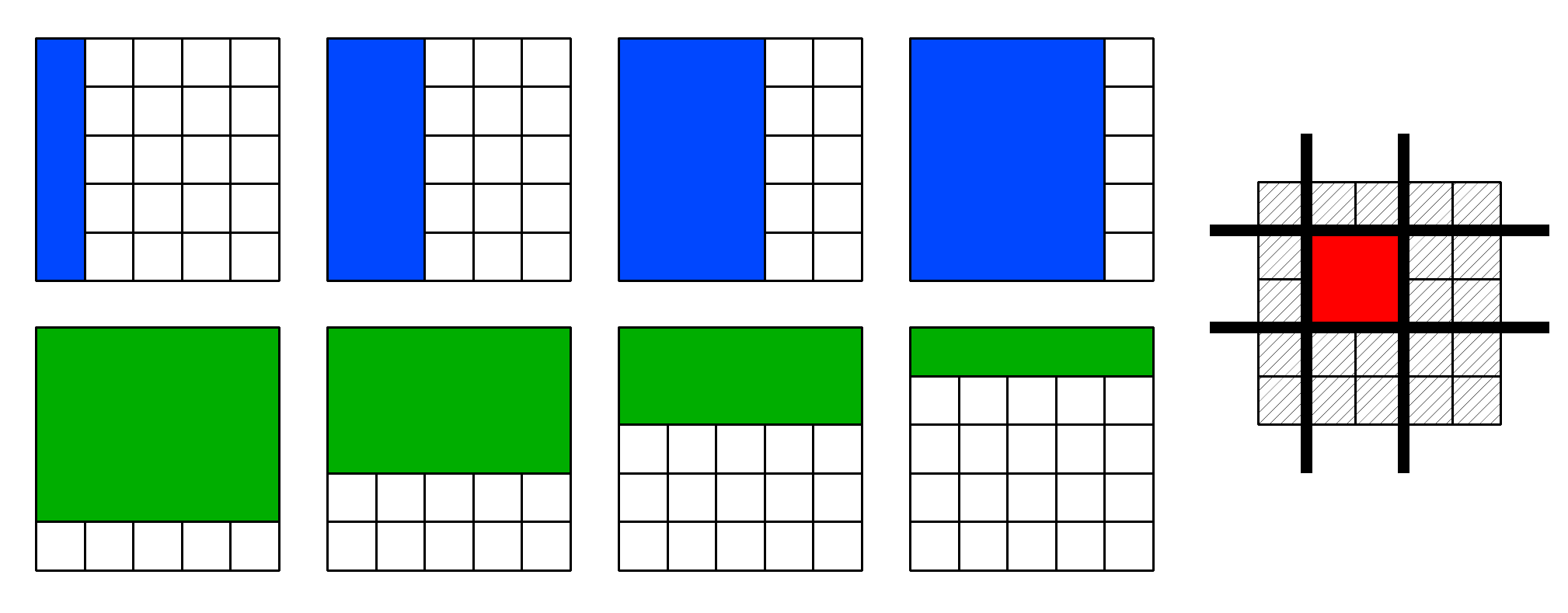}
\caption{\small Set $\G$ of overlapping groups defining the norm proposed by \citet{jenatton2011structured} (set in blue or green and their complements) and an example of  corresponding induced sparsity patterns (in red), respectively for interval patterns in 1D (left) and for rectangular patterns in 2D (right).}
\label{fig:groups}
\end{figure}

  \begin{figure}
\begin{center}
\includegraphics[width=5cm]{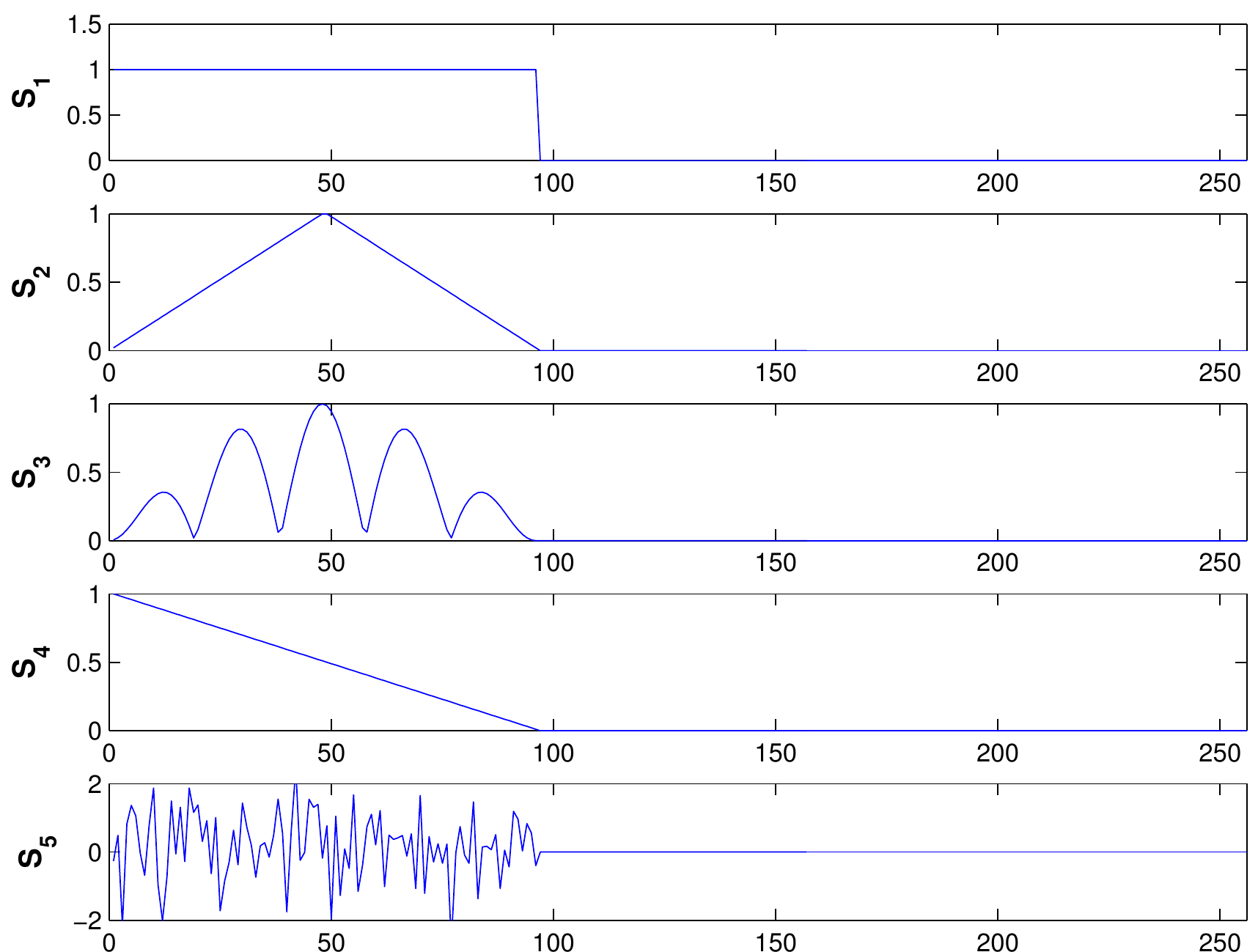}
\end{center}
\caption{\small Examples of the shape of the signals used to define the amplitude of the coefficients of $w$ on the support. Each plot represents the value of $w_i$ as a function of $i$. The first ($w$ constant on the support), third ($w_i=g(c\,i)$ with $g: x\mapsto |\sin(x) \sin(5x)|$) and last signal ($w_i \overset{\text{\tiny{i.i.d.}}}{\sim} \mathcal{N}(0,1)$) are the ones used in reported results.}
\label{fig:signals}
\end{figure}

Moreover, the authors showed that with a weighting scheme leading to a norm of the form $w \mapsto \sum_{B \in \mathcal{G}} \|w_B \circ d^B\|$, where $\circ$ denotes the Hadamard product and $d^B \in \RR^d_+$ is a certain vector of weights designed specifically for these case\footnote{We refer the reader to the paper for the details.} it is possible to obtain compelling empirical results in terms of support recovery, especially in the 1D case.

\textbf{Interval supports.} From the point of view of our work, that is, approaching the problem in terms of combinatorial functions, for supports constrained to be intervals, it is natural to consider the range function as a possible form of penalty: $F_0(A):={\rm range}(A)=i_{\max}(A)-i_{\min}(A)+1$. Indeed the range function assigns the same penalty to sets with the same range, regardless of whether these sets are connected or have ``holes"; this clearly favors intervals since they are exactly the sets with the largest support for a given value of the penalty.  Unfortunately, as discussed in the Example~\ref{ex:range} of Section~\ref{sec:lce}, the combinatorial lower envelope of the range function is $A \mapsto |A|$, the cardinality function, which implies that $\Omega_p^{F_0}$ is just the $\ell_1$-norm: in this case, the structure implicitly encoded in $F_0$ is lost through the convex relaxation.

However, as mentioned by \cite{bach2010structured} and discussed in Example~\ref{ex:submodular_range} the function $F_r$ defined by $F_r(A)=d-1+{\rm range}(A)$ for $A \neq \varnothing$ and $F(\varnothing)=0$ is submodular, which means that $\Omega^{F_r}_p$ is a tight relaxation and that regularizing with it leads to tractable convex optimization problems.

\textbf{Rectangular supports.} For the case of rectangles on the grid, a good candidate is the function $F_2$ with $F_2(A)={F_r}(\Pi_1(A))+{F_r}(\Pi_2(A))$
with $\Pi_i(A)$ the projection of the set $A$ along the $i$th axis of the grid.
 
This makes of $\Omega_p^{F_r}$ and $\Omega_p^{F_2}$ two good candidates to estimate a vector $w$ whose support matches respectively the two described a priori.

\subsection{Methodology}

We consider a simple regression setting in which $w \in \RR^d$ is a vector such that $\supp(w)$ is either an interval on $[1,d]$ or a rectangle on a fixed $2D$ grid. We draw the design matrix $X \in \RR^{n \times d}$ and a noise vector $\eps \in \RR^n$ both with i.i.d. standard Gaussian entries and compute $y=X w+\epsilon$. We then solve problem~(\ref{eq:objective}), with $\Omega$ chosen in turn to be the $\ell_1$-norm (Lasso), the elastic net, the norms $\Omega^F_p$ for $p \in\{2,\infty\}$ and $F$ chosen to be $F_r$ or $F_2$ in $1D$ and $2D$ respectively; we consider also the overlapping $\ell_1/\ell_2$-norm proposed by \citet{jenatton2011structured} and the weighted overlapping $\ell_1/\ell_2$-norm proposed by the same authors, i.e., $\Omega(w)=\sum_{B \in \mathcal{G}} \|w_B \circ d^B\|_2$ with the same notations as before\footnote{Note that we do not need to compare with an $\ell_infty$ counterpart of the unweighted norm considered in \citet{jenatton2011structured} since for $p=\infty$ the unweighted $\ell_1/\ell_\infty$ norm defined with the same collection $\G$ is exactly the norm $\Omega_\infty^{F_r}$: this follows from the form of $F_r$ as defined in Example~\ref{ex:submodular_range} and the preceding discussion.}.

We assess the estimators obtained through the different regularizers both in terms of support recovery and in terms of mean-squared error in the following way: assuming that held out data permits to choose an optimal point on the regularization path obtained with each norm,   we determine along each such path, the solution which either has a support with minimal Hamming distance to the true support or the solution which as the best $\ell_2$ distance, and we report the corresponding distances as a function the sample size on Figures~\ref{fig:results_1D} and \ref{fig:results_2D} respectively for the 1D and the 2D case.

Finally, we assess the incidence of the fluctuation in amplitude of the coefficients in the vector $w$ generating the data: we consider different cases among which:
\BIT
\item[(i)] the case where $w$ has a constant value on the support,
\item[(ii)] the case where $w_i$ varies as a modulated cosine, with $w_i=g(c i)$ for $c$ a constant scaling and $g: x \mapsto |\cos(x)\cos(5x)|$ 
\item[(iii)] the case where $w_i$ is drawn i.i.d. from a standard normal distribution.
\EIT
These cases (and two others for which we do not report results) are illustrated on Figure~\ref{fig:signals}.

\subsection{Results}

Results reported for the Hamming distances in the left columns of Figures~\ref{fig:results_1D} and \ref{fig:results_2D} show that the norms $\Omega_2^{F_r}$ and $\Omega_2^{F_2}$ perform quite well for support recovery overall and tend to outperform significantly their $\ell_\infty$ counterpart in most cases. 
In 1D, several norms achieve reasonably small Hamming distance, including the $\ell_1$-norm, the norm $\Omega_2^{F_r}$ and  the weighted overlapping $\ell_1/\ell_2$-norm although the latter clearly dominates for small values of $n$.

In 2D,  $\Omega_2^{F_2}$ leads clearly to smaller Hamming distances than other norms for the larger values of $n$, while is outperformed by the $\ell_1$-norm for small sample sizes. It should be noted that neither $\Omega_\infty^{F_2}$ nor the weighted overlapping $\ell_1/\ell_2$-norm that performed so well in 1D achieve good results.

The performance of the $\ell_2$ relaxation tends to be comparatively better when the vector of parameter $w$ has entries that vary a lot, especially when compared to the $\ell_\infty$ relaxation. Indeed, the choice of the value of $p$ for the relaxation can be interpreted as 
encoding a prior on the joint distribution of the amplitudes of the $w_i$: as discussed before, and as illustrated in~\cite{bach2010structured} the unit balls for the $\ell_\infty$ relaxations display additional ``edges and corners" that lead to estimates with clustered values of $|w_i|$, corresponding to an priori that many entries in $w$ have identical amplitudes. More generally, large values of $p$ correspond to the prior that the amplitude varies little while their vary more significantly for small $p$.

The effect of this other type of a priori encoded in the regularization is visible when considering the performance in terms of $\ell_2$ error.
Overall, both in 1D and 2D all methods perform similarly in $\ell_2$ error, except that when $w$ is constant on the support, the $\ell_\infty$ relaxations $\Omega_\infty^{F_r}$ and $\Omega_\infty^{F_2}$ perform significantly better, and this is the case most likely because the additional ``corners" of these norms induce some pooling of the estimates of the value of the $w_i$, which improves their estimation. By contrast it can be noted that when $w$ is far from constant the $\ell_\infty$ relaxations tend to have slightly larger least-square errors, while, on contrary, the $\ell_1$-regularisation tends to be among the better performing methods.

\begin{figure}
\begin{tabular}{cc}
\includegraphics[width=0.45\tw]{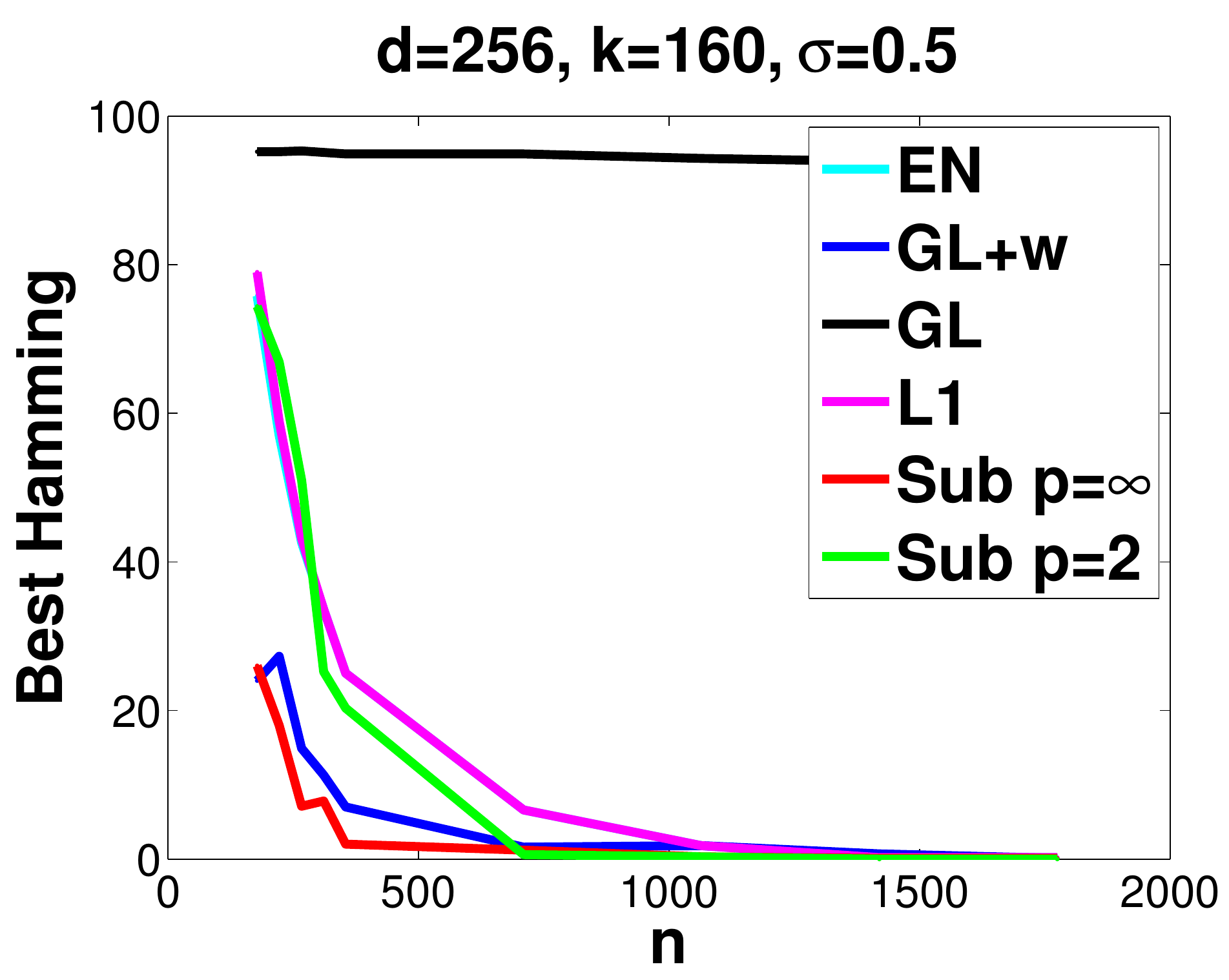} &
\includegraphics[width=0.45\tw]{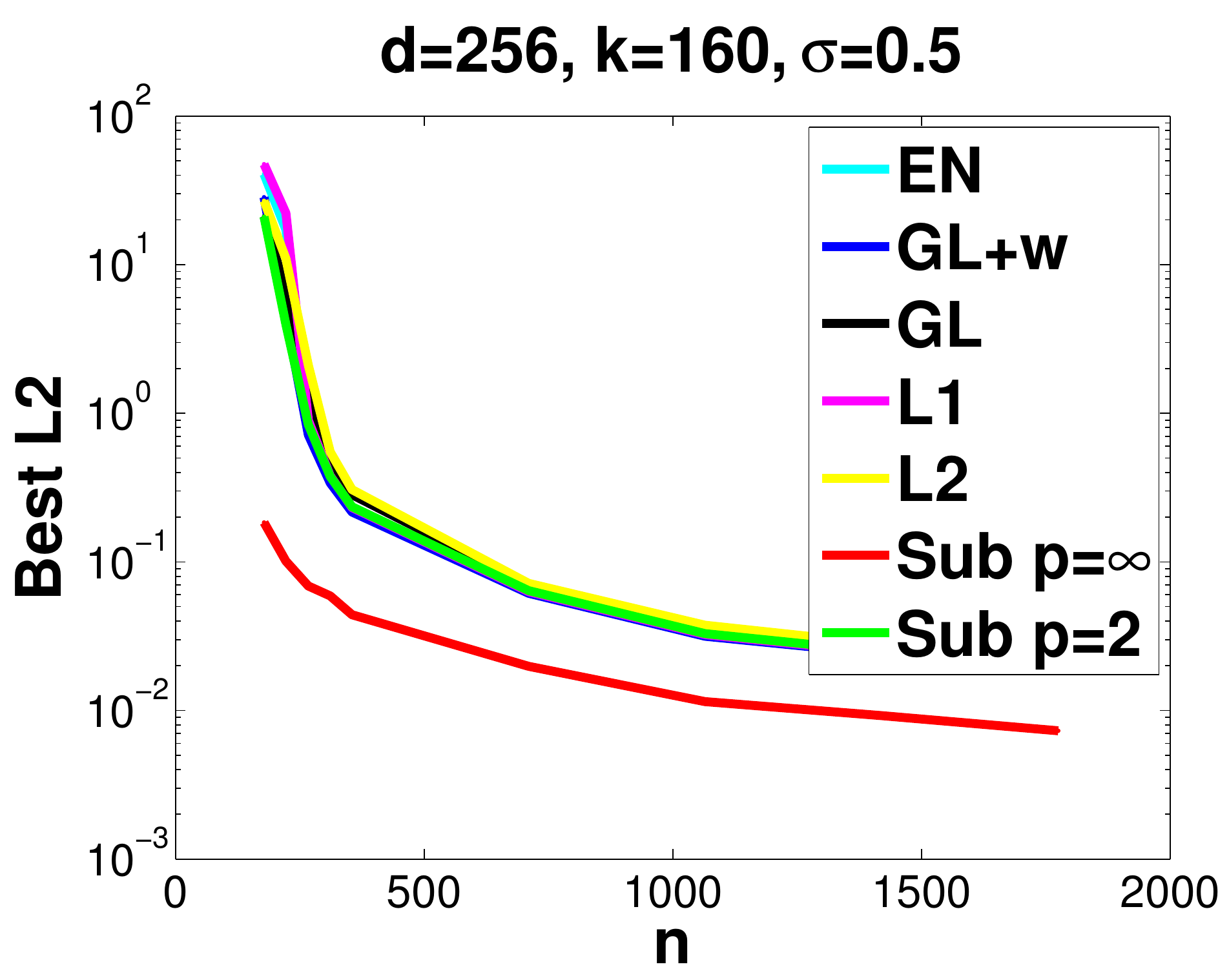} \\
\includegraphics[width=0.45\tw]{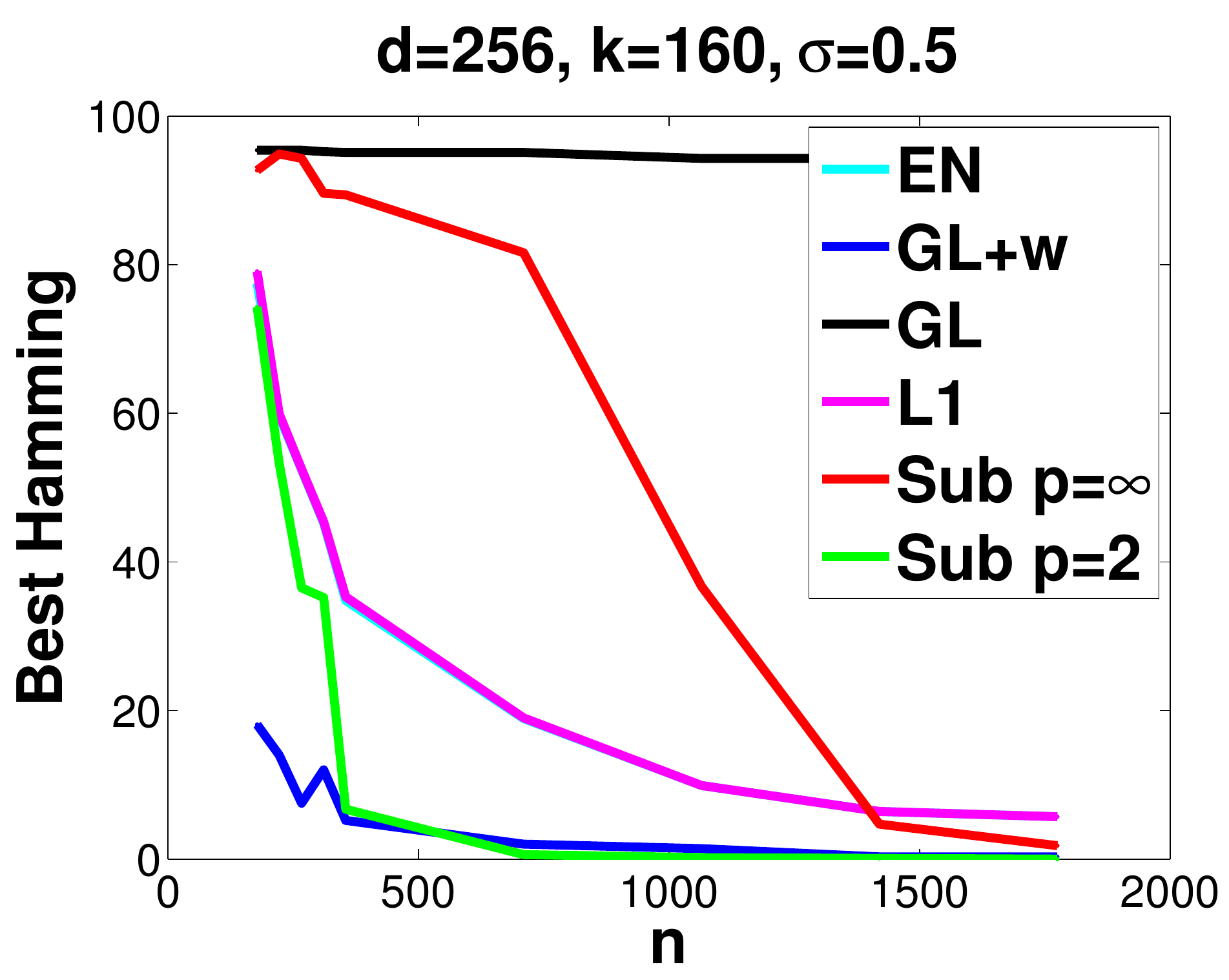} &
\includegraphics[width=0.45\tw]{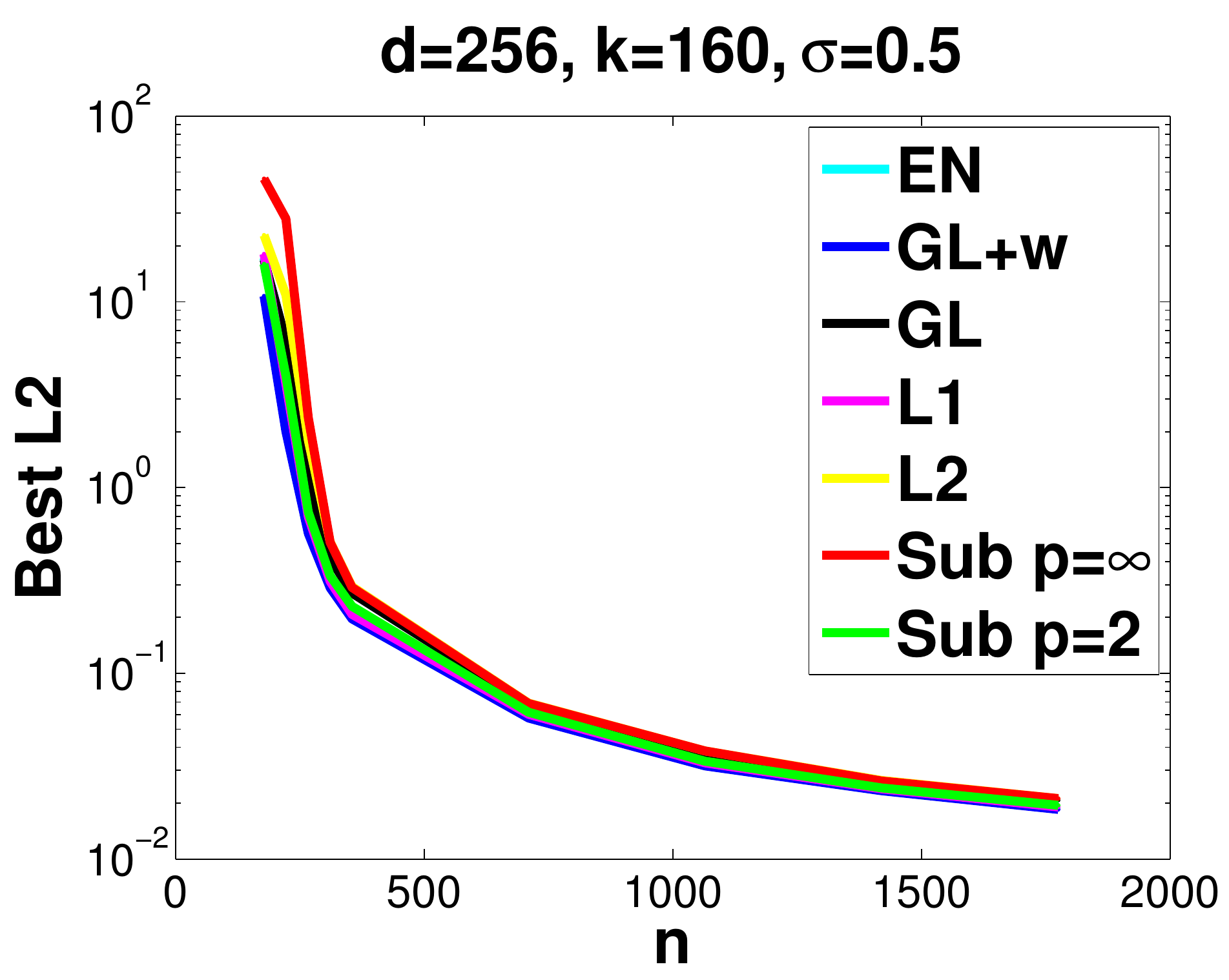} \\
\end{tabular}
\caption{Best Hamming distance (left column) and best least square error (right column) to the true parameter vector $w^*$, among all vectors along the regularization path of a least square regression regularized with a given norm, for different patterns of values of $w^*$.
The different regularizers compared include the Lasso (L1), Ridge (L2), the elastic net (EN), the unweighted (GL) and weighted (GL+w) $\ell_1/\ell_2$ regularizations proposed by \citet{jenatton2011structured}, the norms $\Omega^F_2$ (Sub $p=2$) and $\Omega^F_\infty$ (Sub $p=\infty$) for a specified function $F$. (first row) Constant signal supported on an interval, with an a priori encoded by the combinatorial function $F: A \mapsto d-1+{\rm range}(A)$.
(second row) Same setting with a signal $w^*$ supported by an interval consisting of coefficients $w_i^*$ drawn from a standard Gaussian distribution.  
 In each case, the dimension is $d=256$, the size of the true support is $k=160$ , the noise level is $\sigma=0.5$ and signal amplitude $\|w\|_\infty=1$.}
\label{fig:results_1D}
\end{figure}

\begin{figure}
\begin{tabular}{cc}
\includegraphics[width=0.45\tw]{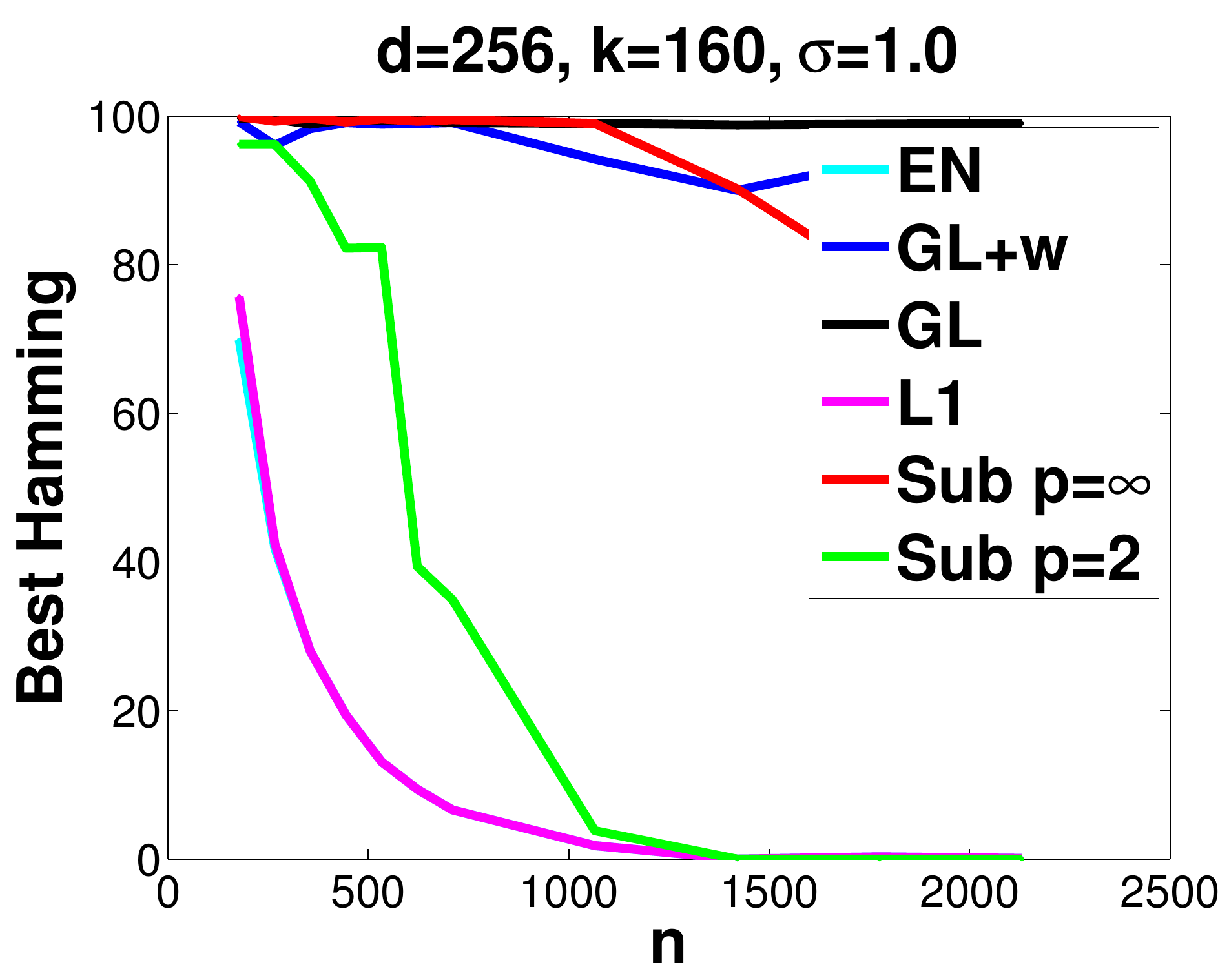} &
\includegraphics[width=0.45\tw]{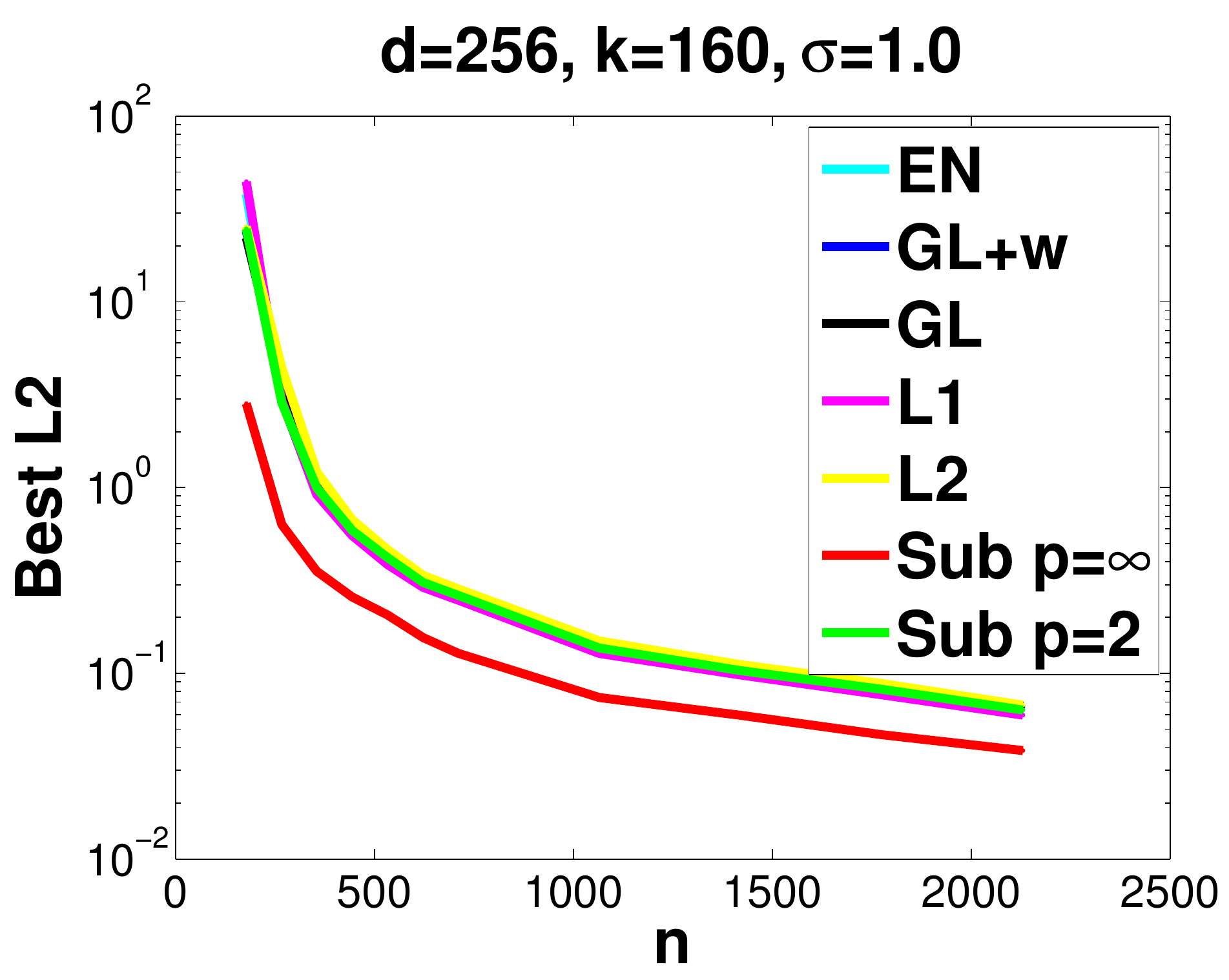}\\
\includegraphics[width=0.45\tw]{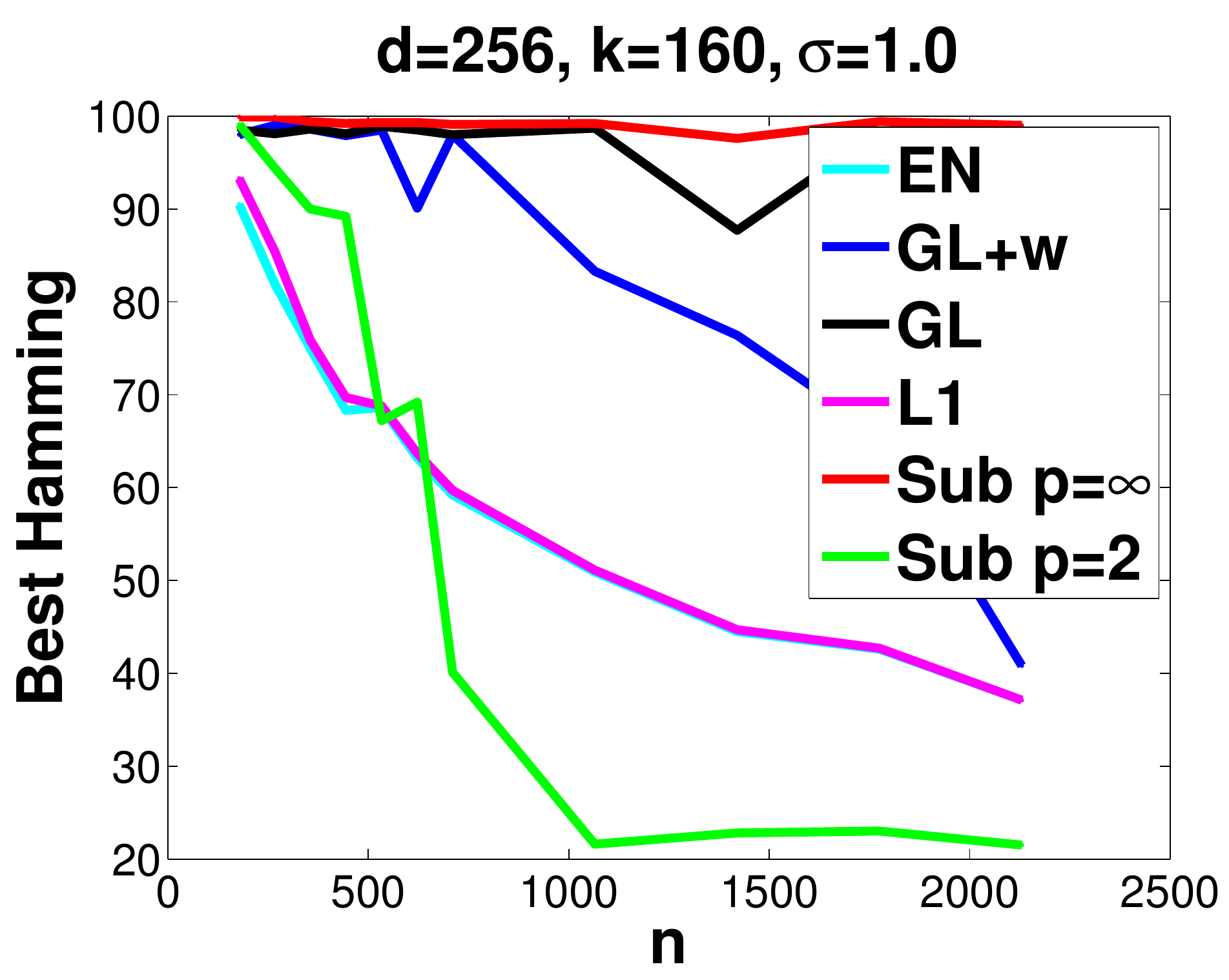} &
\includegraphics[width=0.45\tw]{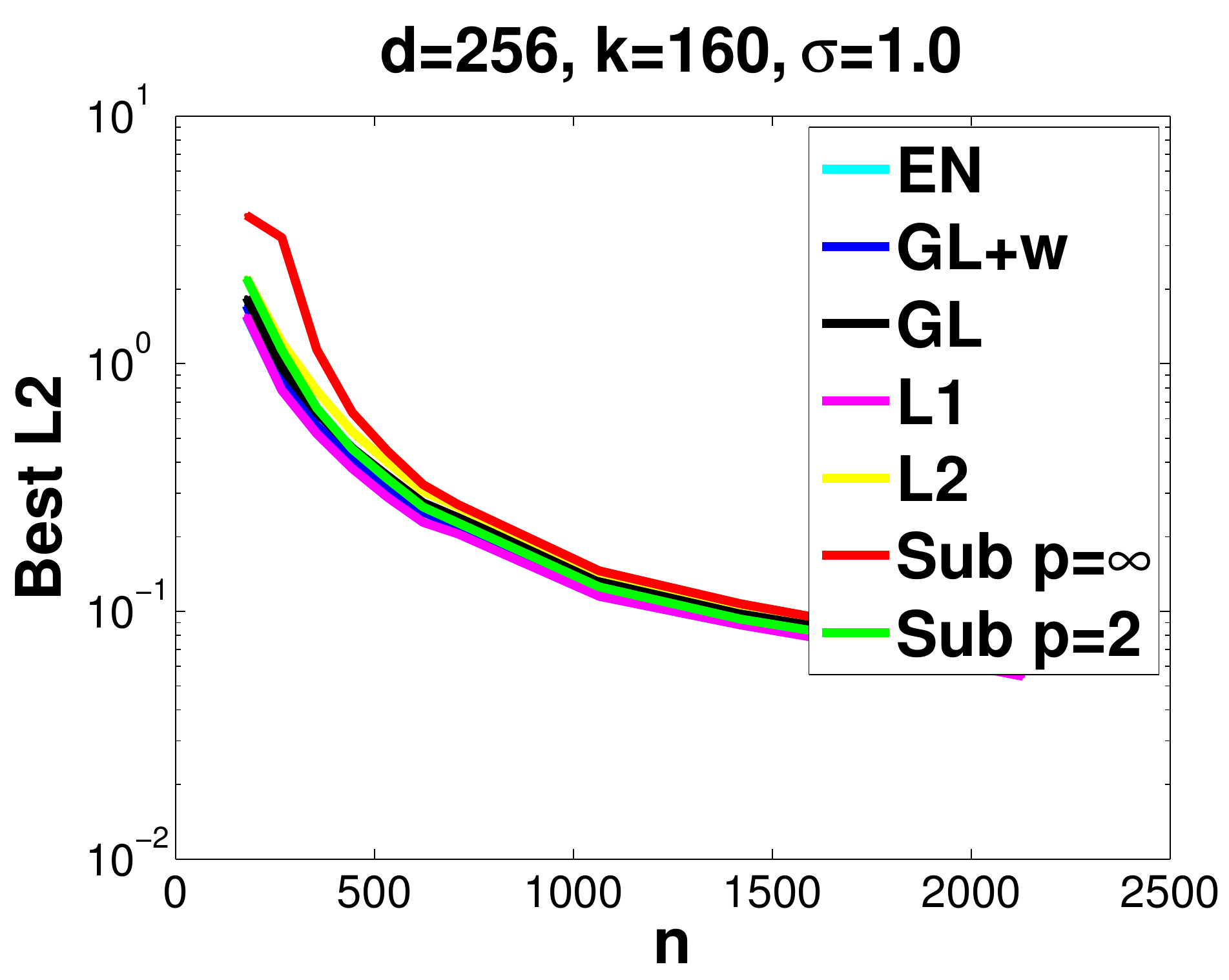} \\
\includegraphics[width=0.45\tw]{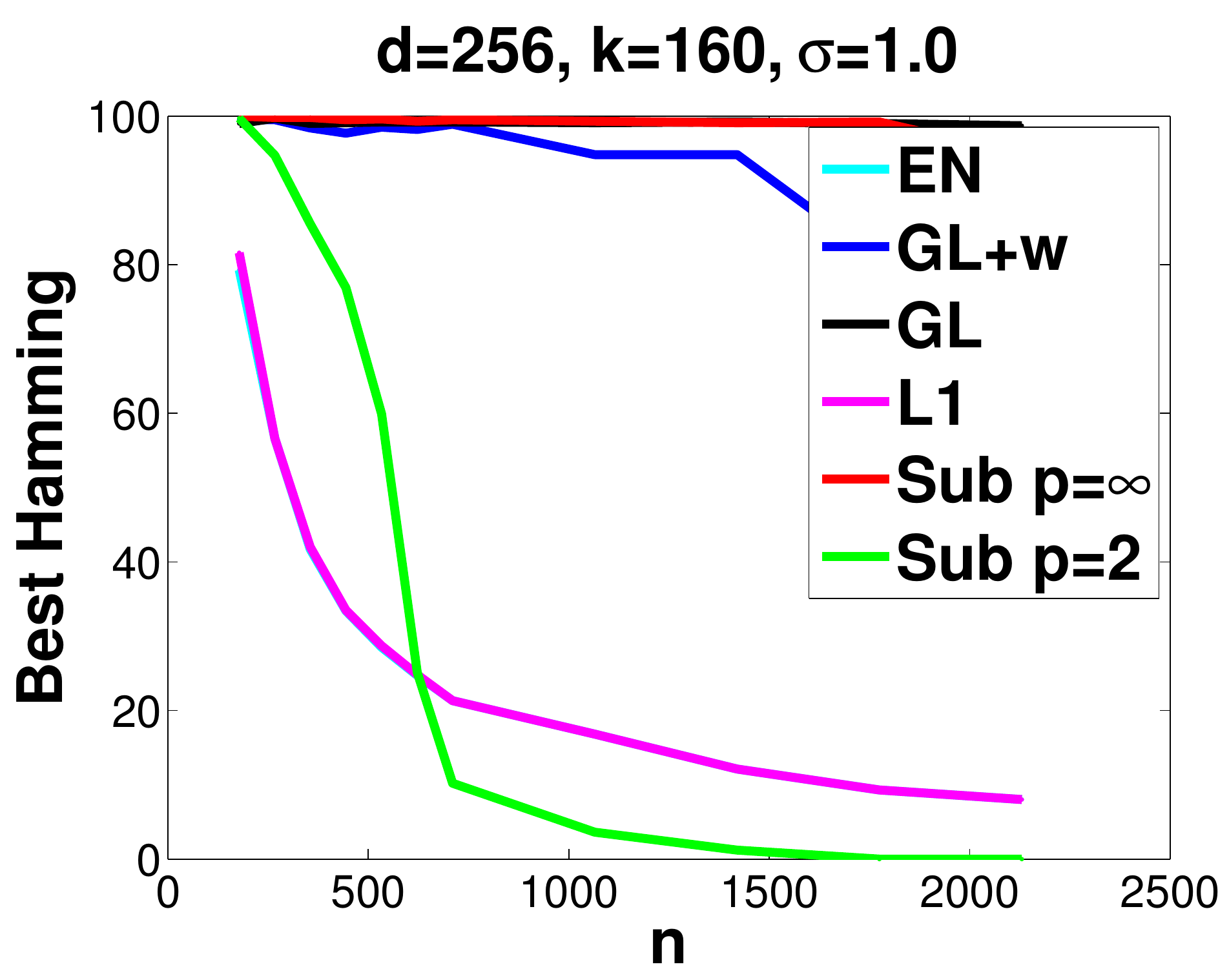} &
\includegraphics[width=0.45\tw]{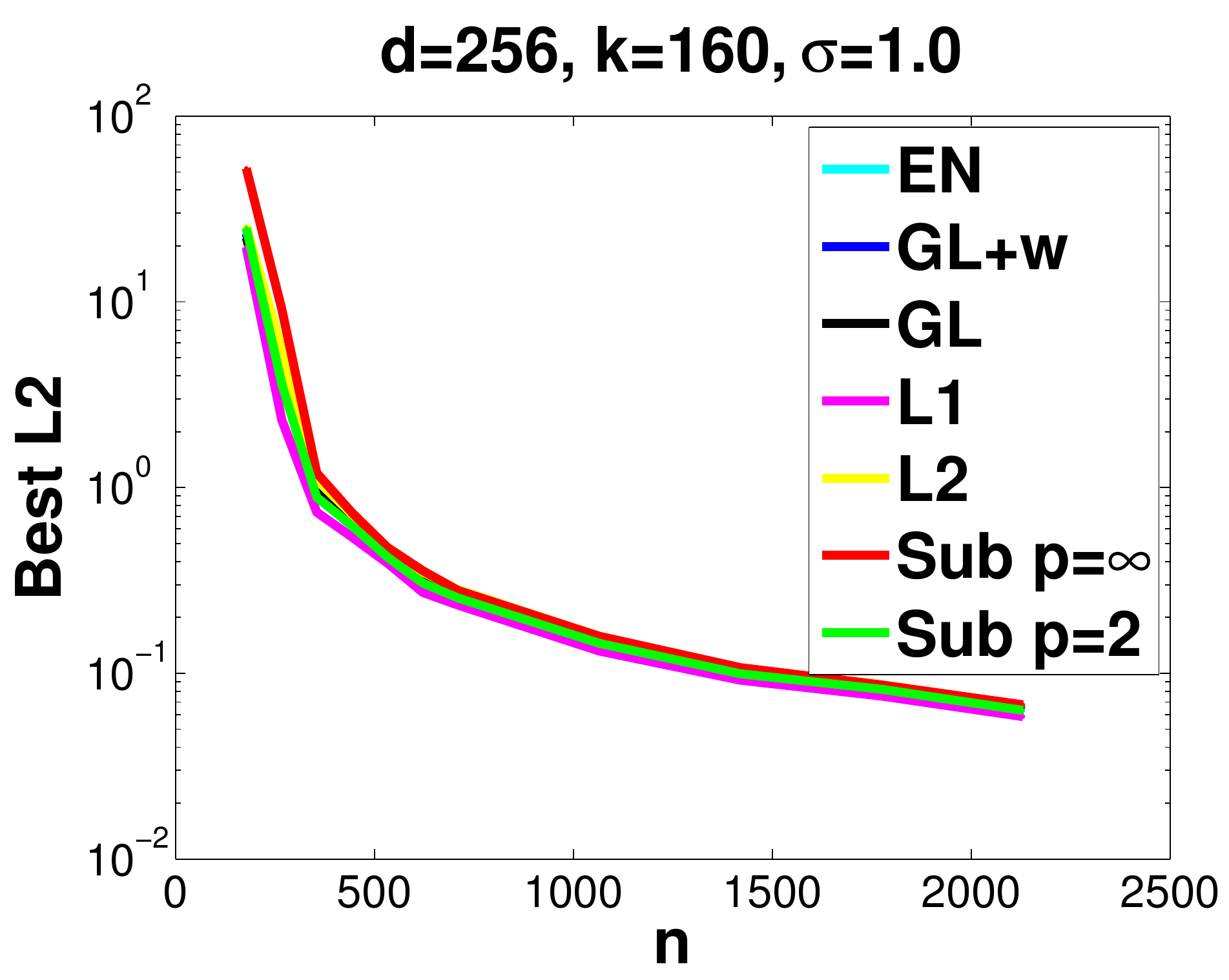} \\
\end{tabular}
\caption{\small Best Hamming distance (left column) and best least square error (right column) to the true parameter vector $w^*$, among all vectors along the regularization path of a least square regression regularized with a given norm, for different patterns of values of $w^*$. The regularizations compared include the Lasso (L1), Ridge (L2), the elastic net (EN), the unweighted (GL) and weighted (GL+w) $\ell_1/\ell_2$ regularizations proposed by \citet{jenatton2011structured}, the norms $\Omega^F_2$ (Sub $p=2$) and $\Omega^F_\infty$ (Sub $p=\infty$) for a specified function $F$. 
Parameter vectors $w^*$ considered here have coefficients that are supported by a rectangle
on a grid with size $d_1\times d_2$ with $d=d_1d_2$.
(first row) Constant signal supported on a rectangle with an a priori encoded by the combinatorial function $F: A \mapsto d_1+d_2-4+{\rm range}(\Pi_1(A))+{\rm range}(\Pi_2(A))$.
(second row) Same setting with coefficients of $w$ on the support given as $w^*_{i_1i_2}=g(c\, i_1)g(c\, i_2)$ for $c$ a positive constant and  $g: x \mapsto |\cos(x)\cos(5x)|$.
(third row) Same setting with coefficients $w^*_{i_1i_2}$ drawn from a standard Gaussian distribution.  
 In each case, the dimension is $d=256$, the size of the true support is $k=160$ , the noise level is $\sigma=1$ and signal amplitude $\|w\|_\infty=1$.}   
\label{fig:results_2D}
\end{figure}

\section{Conclusion}

We proposed a family of convex norms defined as relaxations of penalizations that combine a combinatorial set-function with an $\ell_p$-norm. Our formulation allows to recover in a principled way classical sparsity inducing regularizations such as $\ell_1$, $\ell_1/\ell_p$-norms or $\ell_p/\ell_1$-norms. In addition, it establishes the the latent group Lasso is the tightest relaxation of block-coding penalties.

There are several directions for future research. First, it would be of interest to determine for which combinatorial functions beyond submodular ones, efficient algorithms and consistency results can be established. Then a sharper analysis of the relative performance of the estimators using different levels of a priori would be needed to answer question such as: When is using a structured a priori likely to yield better estimators? When could it degrade the performance? What is the relation to the performance of an oracle given a specified structured a priori?

\subsection*{Acknowledgements}
The authors acknowledge funding from the European Research Council grant SIERRA, project 239993, and would like to thank Rodolphe Jenatton and Julien Mairal for stimulating discussions.

\bibliographystyle{plainnat}
\bibliography{submodlp}

\appendix
\section{Form of primal norm}
\label{sec:form_primal}
 We provide here a proof of lemma~\ref{lem:duality} which we first recall:
\begin{lemmann}[\ref{lem:duality}] $\Omega_p$ and $\Omega_p^*$ are dual to each other.
\end{lemmann} 
\begin{proof}
Let $\omega^A_p$ be the function\footnote{Or gauge function to be more precise.} defined by $\omega^A_p(w)=F(A)^{1/q}\, \|w_A\|_p \, \iota_{\{v | Supp(v) \subset A\}}(w)$ with $\iota_B$ the indicator function taking the value $0$ on $B$ and $\infty$ on $B^c$. Let $K^A_p$ be the set $K^A_p=\{s \mid \|s_A\|_q^q \leq F(A)\}$. 
By construction, $\omega^A_p$ is the \emph{support function} of $K^A_p$ \citep[see][sec.13]{Rockafellar1970Convex}, i.e. $\omega^A_p(w)=\max_{s \in K^A_p} w^\top s$.
By construction we have $\{s \mid \Omega_p^*(s) \leq 1\}=\cap_{A \subset V} K^A_p$. But this implies that $\iota_{\{s| \Omega_p^*(s) \leq 1\}}=\sum_{A \subset V} \iota_{K^A_p}$. Finally, by definition of Fenchel-Legendre duality,
$$\Omega_p(w)=\max_{w \in \RR^d} w^\top s - \sum_{A \subset V} \iota_{K^A_p}(s),$$ or in words $\Omega_p$ is the Fenchel-Legendre dual to the sum of the indicator functions $\iota_{K^A_p}$. But since the Fenchel-Legendre dual of a sum of functions is the \emph{infimal convolution} of the duals of these functions \citep[see][Thm. 16.4 and Corr. 16.4.1, pp. 145-146]{Rockafellar1970Convex}, and since by definition of a support function $\big (\iota_{K^A_p} \big )^*=\omega^A_p$, then $\Omega_p$ is the \emph{infimal convolution}  of the functions $\omega^A_p$, i.e.
$$\Omega_p(w)=\inf_{(v^A \in \RR^d)_{A \subset V}} \sum_{A \subset V} \omega^A_p(\v^A) \quad \st \quad w=\sum_{A \subset V} v^A,$$
which is equivalent to formulation (\ref{eq:lgl}).
See \citet{obozinski2011group} for a more elementary proof of this result.
\end{proof} 
 
\section{Example of the Exclusive Lasso}
\label{sec:exclusive_plus}

We showed in Section~\ref{sec:exclusive} that the $\ell_p$ exclusive Lasso norm, also called  $\ell_p/\ell_1$-norm, defined by the mapping $w \mapsto  \Big ( \sum_{G \in \G} \|w_{G}\|_1^p  \Big )^{1/p}$, for some partition $\G$, is a norm $\Omega_p^F$ providing the $\ell_p$ tightest convex p.h.\ relaxation in the sense defined in this paper of a certain combinatorial function $F$. A computation of the lower combinatorial envelope of that function $F$ yields the function $\Fl: A \mapsto \max_{G \in \G} |A \cap G|$.
 
This last function is also a natural combinatorial function to consider and by the properties of a LCE it has the same convex relaxation.  It should be noted that it is however less obvious
to show directly that $\Omega^{\Fl}_p$ is the $\ell_p/\ell_1$ norm...

We thus show a direct proof of that result since it illustrates how the results on LCE and UCE can be used to analyze norms and derive such results.

\begin{lemma}
Let $\G=\{G_1,\ldots,G_k\}$ be a partition of $V$. For $F: A \mapsto \max_{G \in \G} |A \cap G|$, we have\\ $\Omega^F_\infty(w)=\max_{G \in \G} \|w_G\|_1$.
\end{lemma}
\begin{proof}
Consider the function $f: w \mapsto \max_{G \in \G} \|w_G\|_1$ and the set function $F_0: A \mapsto f(1_A)$. We have $F_0(A)=\max_{G \in \G} \|1_{A \cap G}\|_1=F(A)$. 
But by Lemma~\ref{lch_ext}, this implies that $f(w) \leq \Omega^F_\infty(w)$ since $f=f(|\cdot|)$ is convex positively homogeneous and coordinatewise non-decreasing on $\RR^d_+$.
We could remark first that since $F(A)=f(1_A) \leq \Omega^F_\infty(1_A) \leq F(A)$, this shows that $F=\Fl$ is a lower combinatorial envelope. 
Now note that 
$$(\Omega^F_\infty)^*(s)=\max_{A \subset V, A \neq \varnothing} \min_{G \in \G} \frac{\|s_A\|_1}{|A \cap G|}\geq  \max_{A \subset V,\: |A \cap G|=1,\, G \in \G} \|s_A\|_1=\sum_{G \in \G} \max_{i \in G} |s_i|=\sum_{G \in \G} \|s_G\|_\infty.
$$
This shows that $(\Omega^F_\infty)^*(s) \geq \sum_{G \in \G} \|s_G\|_\infty$, which implies for dual norms that
$\Omega^F_\infty(w) \leq f(w)$. Finally, since we showed above the opposite inequality $\Omega_\infty^F=f$ which shows the result.
 \end{proof}
 
\section{Properties of the norm $\Omega^F_p$ when $F$ is submodular}

In this section, we first derive upper bounds and lower bounds for our norms, as well as a local formulation as a sum of $\ell_p$-norms on subsets of indices.

\subsection{Some important inequalities.} 

We now derive inequalities which will be useful later in the theoretical analysis.
By definition, the dual norm satisfies the following inequalities:
\BEQ
\frac{\|s\|_\infty}{M^\frac{1}{q}} \leq \max_{k \in V} \frac{\|s_{\{k\}}\|_q}{F(\{k\})^\frac{1}{q}} \leq \Omega^*_p(s)=\max_{A \subset V, A \neq \varnothing} \frac{\|s_A\|_q}{F(A)^\frac{1}{q}} \leq \frac{\|s\|_q}{\min_{A \subset V, A \neq \varnothing} F(A)^{\frac{1}{q}}} \leq \frac{\|s\|_q}{m^\frac{1}{q}},
\EEQ
for $m=\min_{k \in V} F(\{k\})$ and $M=\max_{k \in V} F(\{k\})$.
These inequalities imply immediately inequalities for $\Omega_p$ (and therefore for $f$ since for $\eta \in \RR^d_+, \: f(\eta)=\Omega_\infty(\eta))$:
$$ m^{1/q}  \| \w \|_p \leqslant \Omega_p(\w) \leqslant M^{1/q}  \| \w \|_1.$$
We also have $\Omega_p(w) \leqslant F(V)^{1/q}\| w\|_p$, using the following lower bound for the dual norm:
$\Omega_p^\ast(s) \geqslant \frac{ \| s\|_p}{F(V)^{1/q}}$.

Since by submodularity, we in fact have $M = \max_{A, k \notin A} F(A \cup \{k\} ) - F(A)$, it makes sense to introduce $\tilde{m}
= \min_{A , k , F(A \cup \{k\} ) > F(A) }F(A \cup \{k\} ) - F(A) \leq m$.
Indeed, we consider in \mysec{theory} the norm $\Omega_{p,J}$ (resp. $\Omega_p^J$) associated with \emph{restrictions} of $F$ to $J$ (resp.  \emph{contractions} of $F$ on $J$) and it follows from the previous inequalities that for all $J \subset V$, we have:
$${ \tilde{m}}^{1/q}  \| \w \|_p \leqslant  { {m}}^{1/q}  \| \w \|_p \leqslant \Omega_{p,J}(\w) \leqslant M^{1/q}  \| \w \|_1 \text\qquad \text{and} \qquad {\tilde{m}}^{1/q}  \| \w \|_p \leqslant \Omega^J_p(\w) \leqslant M^{1/q}  \| \w \|_1.$$

\subsection{Some optimality conditions for $\eta$.}  
While exact necessary and sufficient conditions for $\eta$ to be a solution of \eq{eta-trick} would be tedious to formulate precisely, we provide three necessary and two sufficient conditions, which together characterize a non-trivial subset of the solutions, which will be useful in the subsequent analysis.

\begin{proposition}[Optimality conditions for $\eta$] Let $F$ be a non-increasing submodular function.
Let $p>1$ and $w \in \rb^d$, $K = \supp(w)$ and $J$ the smallest stable set containing $K$. Let $H(w)$ the set of minimizers of \eq{eta-trick}. Then,

(a) the set $\{ \eta_K, \ \eta \in H(w) \}$ is a singleton with strictly positive components, which we denote $\{\eta_K(w)\}$, i.e., \eq{eta-trick} uniquely determines $\eta_K$.

(b) For all $\eta \in H(w)$, then $\eta_{J^c} = 0$.

(c) If $A_1 \cup \dots \cup A_m$ are the ordered level sets of $\eta_K$, i.e., $\eta$ is constant on each $A_j$ and the values on $A_j$ form a strictly decreasing sequence, then $F( A_1 \cup \cdots   \cup A_j) - F( A_1 \cup \cdots   \cup A_{j-1}) > 0$ and the value on $A_j$ is equal to
$\eta^{A_j}(\w) =  \frac{\| \w_{A_j} \|_p}{ [ F( A_1 \cup \cdots   \cup A_j) - F( A_1 \cup \cdots   \cup A_{j-1})]^{1/p}}.$

(d) If $\eta_K$ is equal to $\eta_K(w)$, $\max_{k \in  J \backslash K } \eta_k \leqslant \min_{k \in K} \eta_k(w)$, and $\eta_{J^c}=0$, then $\eta \in H(w)$.

(e) There exists $\eta \in H(w)$ such that 
$\frac{ \min_{i \in K } |w_i|  }{M^{1/p}}  \leqslant \min_{j \in J } \eta_j \leqslant 
\max_{j \in J } \eta_j \leqslant \frac{ \| w \|_p }{m^{1/p}}.$

\end{proposition}
\begin{proof}
(a) Since $f$ is non-decreasing with respect to each of its argument, for any $\eta \in H(w)$, we have $\eta' \in H(w)$ for $\eta'$ defined through $\eta'_K=\eta_K$ and $\eta_{K^c}=0$. The set of values of $\eta_K$ for $\eta \in H(w)$ is therefore the set of solutions problem (\ref{eq:eta-trick}) restricted to $K$. The latter problem has a unique solution as a consequence of the strict convexity on $\rb_+^\ast$ of $\eta_j \mapsto \frac{|w_j|^p}{\eta_j^{p-1}}$.

(b) If there is $j \in J^c$ such that $\eta \in H(w)$ and $\eta_j \neq 0$, then (since $w_j=0$) because $f$ is non-decreasing with respect to each of its arguments, we may take $\eta_j$ infinitesimally small and all other $\eta_k$ for $k \in K^c$ equal to zero, and we have $f(\eta) = f_K(\eta_K(w)) + \eta_j [ F( K \cup \{j \}) - F(K) ] $. Since 
$F( K \cup \{j \}) - F(K)  \geqslant F( J \cup \{j \}) - F(J) > 0$ (because $J$ is stable), we have $f(\eta) > f_K(\eta_K(w)) $, which is a contradiction.

(c) Given the ordered level sets, we have $f(\eta) = \sum_{j=1}^m \eta^{A_j} [ F( A_1 \cup \cdots   \cup A_j) - F( A_1 \cup \cdots \cup   A_{j-1})]$, which leads to a closed-form expression $\eta^{A_j}(\w) =  \frac{\| \w_{A_j} \|_p}{ [ F( A_1 \cup \cdots   \cup A_j) - F( A_1 \cup \cdots   \cup A_{j-1})]^{1/p}}.$  If $F( A_1 \cup \cdots   \cup A_j) - F( A_1 \cup \cdots   \cup A_{j-1})=0$, since $\| \w_{A_j} \|_p>0$, we have
 $\eta^{A_j}$ as large as possible, i.e., it has to be equal to $\eta^{A_{j-1}}$, thus it is not a possible ordered partition.

(d) With our particular choice for $\eta$, we have
$\sum_{i \in V} \frac{1}{p}\frac{|w_i|^p}{\eta_i^{p-1}}+\frac{1}{q}f(\eta)
= \Omega_K(w_K)
$. Since we always have $\Omega(w) \geqslant \Omega_K(w_K)$, then $\eta$ is optimal   in \eq{eta-trick}.

(e) We take the largest elements from (d) and bounds the components of $\eta_K$ using (c).
\end{proof}

Note that from property (c),  we can explicit the value of the norm as:
\begin{align}
\Omega_p(\w)&=\sum_{j=1}^k (F(A_1 \cup \ldots \cup A_j)-F(A_1 \cup \ldots \cup A_{j-1}))^{\frac{1}{q}} \|\w_{A_j \backslash A_{j-1}}\|_p\\
&=\Omega_{p,A_1}(\w_{A_1})+\sum_{j=2}^k \Omega_{p,A_j}^{A_{j-1}}(\w_{A_j \backslash A_{j-1}})
\label{eq:norm_form}
\end{align}
where $\Omega_{p,B}^A$ is the norm associated with the contraction on $A$ of $F$ restricted to $B$.

\section{Proof of Proposition~\ref{prop:inequality} (Decomposability)}
Concretely, let
$c=\frac{\tilde{m}}{M}$ with $M=\max_{k \in V} F(\{k\})$ and
$$\tilde{m}=\min_{A,k}\: F(A \cup \{k\} ) - F(A) \:\st \: F(A \cup \{k\} ) > F(A)$$ 

\begin{propositionnn}[\ref{prop:inequality}.\ Weak and local Decomposability]
\label{sec:proof_decomp}
(a) For any set $J$ and any $w \in \rb^d$, we have  $$\Omega(\w) \geq \OJ(\w_J)+\OnJ(\w_{J^c}).$$ 
(b) Assume that $J$ is stable, and $\| \w_{J^c} \|_p \leq c^{1/p} \min_{i \in J} |w_i|$, then $\Omega(\w) = \OJ(\w_J)+\OnJ(\w_{J^c})$.\\
(c) Assume that $K$ is non stable and $J$ is the smallest stable set containing $K$, and that $\| \w_{J^c} \|_p \leq c^{1/p} \min_{i \in K} |w_i|$, then $\Omega(\w) = \OJ(\w_J)+\OnJ(\w_{J^c})$.

\end{propositionnn}
\begin{proof}  
We first prove the first statement (a):
If $\|\s_{A \cap J}\|^p_p \leq F(A \cap J)$ and $\|\s_{A \cap J^c}\|^p_p \leq F(A \cup J) - F(J)$ then
by submodularity we have $\|\s_A\|_p^p \leq F(A \cap J) + F(A \cup J)-F(J) \leq F(A)$.
The submodular polyhedra associated with $F_J$ and $F^J$ are respectively defined by
\begin{align*} 
P(F_J)& =\{\s \in \RR^d,\: \supp(\s) \subset J,  \: \s(A) \leq F(A), \: A \subset J \} \quad \text{and} \\
P(F^J)& =\{\s \in \RR^d,  \:  \supp(\s) \subset J^c,\: \s(A) \leq F(A \cup J) -F(J)\}
\end{align*}
Denoting $s^{\ptpow p}:=(s_1^p,\ldots,s_d^p)$, we therefore have $$\Omega(\w)=\max_{\{|\s^{\ptpow p}| \in P(F) \}} \s^\top \w \geq \max_{\{|\s^{\ptpow p}_J| \in \,  P(F_J), \: |\s^{\ptpow p}_{J^c}| \in \, P(F^J) \}} \s^\top \w=\OJ(\w_J)+\OnJ(\w_{J^c}).$$
 
In order to prove (b), we consider an optimal $\eta_J$ for $w_J$ and $\Omega_J$ and an optimal $\eta_{J^c}$ for $\Omega^J$. Because of our inequalities, and because we have assume that $J$ is stable (so that the value $m$ for $\Omega^J$ is indeed lower bounded by $\tilde{m}$), we have $\|\eta_{J^c} \|_\infty \leqslant \frac{ \| \w_{J^c} \|_p}{\tilde{m}^{1/p}}$. Moreover, we have $\min_{j \in J} \eta_j \geqslant \frac{ \min_{i \in J} |w_i| }{ M^{1/p}  }$ (inequality proved in the main paper). Thus when concatenating $\eta_J$ and $\eta_{J^c}$ we obtain an optimal~$\eta$ for $w$ (since then the \lova extension decomposes as a sum of two terms), hence the desired result.

In order to prove (c), we simply notice that since $F(J) = F(K)$, the value of $\eta_{J \backslash K}$ is irrelevant (the variational formulation does not depend on it), and we may take it equal to the largest known possible value, i.e., one which is largest than $\frac{ \min_{i \in J} |w_i| }{ M^{1/p}  }$, and the same reasoning than for (b) applies.
\end{proof}
Note that when $p=\infty$, the condition in (b) becomes  $\min_{i \in J} |w_i| \geqslant \max_{i \in J^c} |w_i|$, and we recover exactly the corresponding result from \citet{bach2010structured}.
 
\section{Algorithmic results}

\subsection{Proof of Algorithm 1}
\label{sec:proof_alg1}
Algorithm 1 is a particular instance of the decomposition algorithm for the optimization of a convex function over the submodular polyhedron (see e.g. section 6.1 of \citet{bach2011learning})
Indeed denoting $\psi_i(\kappa_i)=\min_{w_i \in \RR} \frac{1}{2} (w_i-z_i)^2+\lambda \kappa_i^{1/q}|w_i|$, the computation of the proximal operator amounts to solving in $\kappa$ the problem $$\max_{\kappa \in \RR^d_+ \cap \mathcal{P}} \sum_{i\in V}\psi_i(\kappa_i)$$
Following the decomposition algorithm, one has to solve first
\begin{eqnarray*}
& & \max_{\kappa \in \RR^d_+}\sum_{i\in V}\psi_i(\kappa_i) \quad \st \quad \sum_{i \in V} \kappa_i=F(V)\\
& = &  \min_{w \in \RR^d}\max_{\kappa \in \RR^d_+} \frac{1}{2} \|w-z\|_2^2+\sum_{i \in V} \kappa_i^{1/q}|w_i|\quad \st \quad \sum_{i \in V} \kappa_i=F(V)\\
& = & \min_{w \in \RR^d} \frac{1}{2} \|w-z\|_2^2+\lambda F(V)^{1/q} \|w\|_p,
\end{eqnarray*}
where the last equation is obtained by solving the maximization problem in $\kappa$, which has the unique solution $\kappa_i=F(V)\frac{|w_i|^p}{\|w\|^p_p}$ if $w \neq 0$ and the simplex of solutions $\{\kappa \in \RR_+^d \mid \kappa(V)=F(V)\}$ for $w=0$.

This is solved in closed form for $p=2$ with $w^*=(\|z\|_2-\lambda \sqrt{F(V)})_+\frac{z}{\|z\|_2}$ if $z\neq0$ and $w^*=0$ else.

In particular since $w^*\propto z$, then $\kappa_i=F(V)\frac{z_i^2}{\|z\|_2^2}$ is always
a solution.
Following the decomposition algorithm, one then has to find the minimizer of the submodular function $A \mapsto F(A)-\kappa(A)$. Then one needs to solve 
$$\min_{\kappa_A \in \RR^{|A|}_+ \cap \mathcal{P}(F_A)} \sum_{i \in A} \psi_i(\kappa_i)
\qquad \text{and}
\qquad \min_{\kappa_{V\backslash A} \in \RR^{|V\backslash A|}_+ \cap \mathcal{P}(F^A)} \sum_{i \in V\backslash A} \psi_i(\kappa_i).
$$
Using the expression of $\psi_i$ and exchanging as above the minimization in $w$ and the maximization in $\kappa$, one obtains directly that these two problems correspond respectively to the computation of the proximal operators of $\Omega^{F_A}$ on $z_A$ and of  the proximal operator of $\Omega^{F^A}$ on $z_{V \backslash A}$.

The decomposition algorithm is proved to be correct in section 6.1 of \citet{bach2011learning}
under the assumption that $\kappa_i \mapsto \psi(\kappa_i)$ is a strictly convex function.
The functions we consider here are not strongly convex, and in particular, as mentioned above the solution in $\kappa$ is not unique in case $w^*=0$. The proof of \citet{bach2011learning}
however goes through using any solution of the maximization problem in $\kappa$.
\subsection{Decomposition algorithm to compute the norm}
\label{sec:decomp_alg}
Applying Algorithm 1 in the special case where $\lambda=0$ yields a decomposition
algorithm to compute the norm itself (see Algorithm~\ref{alg:norm_decomp}).

\begin{algorithm}[htbp]
\label{alg:norm_decomp}
\caption{Computation of $\Omega_p^F(z)$}
\begin{minipage}{.5\tw}
\begin{algorithmic}[1]
\REQUIRE $z \in \RR^d$.
\STATE Let $A=\{j \mid z_j \neq 0\}.$
\IF{$A \neq V$}
\STATE \textbf{return}  $\Omega_p^{F_A}(z_A)$
\ENDIF
\STATE Let $t \in \RR^d$ with $t_i=\frac{|z_i|^p}{\|z\|^p_p} F(V)$
\STATE Find $A$ minimizing the submodular function $F-t$
\IF {$A=V$}
\STATE \textbf{return}  $F(V)^{1/q}\|x\|_p$
\ELSE
\STATE \textbf{return}  $\Omega_p^{F_A}(z_A)+\Omega_p^{F^A}(z_{A^c})$
\ENDIF
\end{algorithmic}
\end{minipage}
\end{algorithm}

\section{Theoretical Results}

In this section, we prove the propositions on consistency, support recovery and the concentration result of Section~\ref{sec:theory}. As there, we consider a fixed design matrix $X \in \rb^{n \times p}$  and $y \in \rb^n$ a vector of random responses. Given $\lambda >0$, we define
$\hat{w}$ as a minimizer of the regularized least-squares cost:
\BEQ
\textstyle
\min_{w \in \rb^d} \textstyle \frac{1}{2n} \| y - X w\|_2^2 + \lambda \Omega(w).
\EEQ

\subsection{Proof of Proposition~\ref{prop:support} (Support recovery)}
\label{sec:proof_recovery} 

  \begin{proof} We follow the proof of the case $p=\infty$ from~\cite{bach2010structured}.
    Let $r = \frac{1}{n} X^\top \varepsilon \in \rb^d$, which is normal with mean zero and covariance matrix $\sigma^2 Q / n$.
We have for any $w \in \rb^p$, $$\Omega(w) \geqslant \Omega_J(w_J) +  \Omega^{J}(w_{J^c})
\geqslant \Omega_J(w_J) +  \rho \,   \Omega_{J^c}(w_{J^c}) \geqslant \rho  \,  \Omega(w).$$ This implies that
$\Omega^\ast(r) \geqslant \rho \max \{ \Omega_J^\ast(r_J) , (\Omega^J)^\ast(r_{J^c}) \}$.

Moreover,
$r_{J^c} - Q_{J^c J}Q_{JJ}^{-1} r_J$ is normal with covariance matrix $$\frac{\sigma^2}{n} ( Q_{J^c J^c} - Q_{J^c J} Q_{JJ}^{-1} Q_{J J^c} ) \preccurlyeq \sigma^2 / n  Q_{J^c J^c}.$$ This implies that with probability larger than $1 - 3 P( \Omega^\ast(r) > \lambda  \rho  \eta/2 )$,
we have  $$\Omega_J^\ast(r_J) \leqslant \lambda/2 \qquad \text{and} \qquad
 (\Omega^J)^\ast( r_{J^c} - Q_{J^c J}Q_{JJ}^{-1} r_J ) \leqslant \lambda  \eta/2.$$ 
 
 We denote by $\tilde{w}$ the unique (because $Q_{JJ}$ is invertible) minimum of  $\frac{1}{2n} \| y - X w\|_2^2 + \lambda \Omega(w)$, subject to $w_{J^c}=0$. $\tilde{w}_J$ is defined through $Q_{JJ} ( \tilde{w}_J - {w_J}^\ast ) - r_J = - \lambda s_J$ where $s_J \in \partial \Omega_J(\tilde{w}_J)$ (which implies that $\Omega_J^\ast(s_J) \leqslant 1$) , i.e., $\tilde{w}_J - w^\ast_J = Q_{JJ}^{-1} ( r_J - \lambda s_J)$.  We have:
 \BEAS
 \| \tilde{w}_J - w^\ast_J  \|_\infty  & \leqslant &  
 \max_{ j \in J } | \delta_j^\top  Q_{JJ}^{-1} ( r_J - \lambda s_J) |
 \\
  & \leqslant &  
 \max_{ j \in J } \Omega_J(   Q_{JJ}^{-1}  \delta_j ) \Omega_J^\ast( r_J - \lambda s_J) |
 \\ 
  & \leqslant &  
 \max_{ j \in J } \|   Q_{JJ}^{-1}  \delta_j  \|_p F(J)^{1-1/p}[ \Omega_J^\ast( r_J) +  \lambda   \Omega_J^\ast( s_J)  ]
\\
& \leqslant &  
 \max_{ j \in J } \kappa^{-1}  |J|^{1/p} F(J)^{1-1/p}[ \Omega_J^\ast( r_J) +  \lambda   \Omega_J^\ast( s_J)  ]
\leqslant 
\frac{3}{2} \lambda |J|^{1/p} F(J)^{1-1/p} \kappa^{-1} .
 \EEAS
 Thus if $2 \lambda |J|^{1/p} F(J)^{1-1/p}  \kappa^{-1}  \leqslant \nu$, then 
 $ \| \tilde{w} - w^\ast \|_\infty \leqslant \frac{3 \nu}{4}$, which implies $\supp(\tilde{w}) \supset \supp(w^\ast)$.
 
 In the neighborhood of $\tilde{w}$, we have an exact decomposition of the norm, hence, to show that $\tilde{w}$ is the unique global minimum, we simply
need to show that since we have $(\Omega^J)^\ast( r_{J^c} - Q_{J^c J}Q_{JJ}^{-1} r_J ) \leqslant \lambda  \eta/2 $, $\tilde{w}$ is the unique minimizer of \eq{objective}. For that it suffices to show that $(\Omega^J)^\ast ( Q_{J^c J} (\tilde{w}_J - w_J^\ast) - r_{J^c} ) < \lambda$. We have:
 \BEAS
 (\Omega^J)^\ast ( Q_{J^c J} (\tilde{w}_J - w_J^\ast) - r_{J^c} )
 & = &  (\Omega^J)^\ast ( Q_{J^c J} Q_{JJ}^{-1} ( r_J - \lambda s_J)- r_{J^c} )
\\
 & \leqslant &  (\Omega^J)^\ast ( Q_{J^c J} Q_{JJ}^{-1}   r_J  - r_{J^c} ) + 
\lambda  (\Omega^J)^\ast (  Q_{J^c J} Q_{JJ}^{-1}  s_J )
\\
 & \leqslant &  (\Omega^J)^\ast ( Q_{J^c J} Q_{JJ}^{-1}   r_J  - r_{J^c} ) + 
\lambda (\Omega^J)^\ast[  ( \Omega_J(   Q_{JJ}^{-1} Q_{Jj} ) )_{j \in J^c} ]
\\
& \leqslant &  \lambda  \eta/2 + \lambda ( 1-\eta) < \lambda,
 \EEAS
 which leads to the desired result.
   \end{proof}

\subsection{Proof of proposition~\ref{prop:high-dim} (Consistency)}
 \label{sec:proof_consistency}
 \begin{proof}
 
   Like for the proof of Proposition~\ref{prop:support}, we have $$\Omega(x) \geqslant \Omega_J(x_J) +  \Omega^{J}(x_{J^c})
\geqslant \Omega_J(x_J) +  \rho  \, \Omega_{J^c}(x_{J^c}) \geqslant \rho  \, \Omega(x).$$ Thus, if we assume
$\Omega^\ast(q) \leqslant  \lambda  \rho  /2$, then  $\Omega_J^\ast(q_J) \leqslant \lambda/2$ and
 $(\Omega^J)^\ast( q_{J^c}) \leqslant \lambda /2 $. Let $\Delta  = \hat{w} - w^\ast$.

We follow the proof from~\cite{bickel_lasso_dantzig} by using the decomposition property of the norm $\Omega$.
We have, by optimality of $\hat{w}$:
$$ \frac{1}{2}\Delta^\top Q \Delta+   \lambda \Omega( w^\ast + \Delta) + q^\top \Delta \leqslant \lambda \Omega( w^\ast + \Delta) + q^\top \Delta  \leqslant    \lambda \Omega( w^\ast)
$$
Using the decomposition property,
$$  \lambda \Omega_J( (w^\ast + \Delta)_J ) + \lambda \Omega^J((w^\ast + \Delta)_{J^c} ) + q_J^\top \Delta_J
+ q_{J^c}^\top \Delta_{J^c}
  \leqslant    \lambda \Omega_J( w^\ast_J),
$$
$$   \lambda \Omega^J( \Delta_{J^c} )  \leqslant    \lambda \Omega_J( w^\ast_J) - \lambda \Omega_J( w^\ast_J + \Delta_J ) 
+ \Omega_J^\ast(q_J) \Omega_J(\Delta_J)
+ 
(\Omega^J)^\ast (q_{J^c}) \Omega^J(  \Delta_{J^c}),\quad \text{and}
$$
$$   ( \lambda - (\Omega^J)^\ast (q_{J^c}) ) \Omega^J( \Delta_{J^c} )  \leqslant    ( \lambda +  \Omega_J^\ast(q_J)   ) \Omega_J( \Delta_J ).$$
Thus $\Omega^J( \Delta_{J^c} )  \leqslant  3 \Omega_J( \Delta_J)$, which implies $
\Delta^\top Q \Delta \geqslant \kappa \|\Delta_J \|_2^2
$ (by our assumption which generalizes the usual $\ell_1$-restricted eigenvalue condition).
Moreover, we have:
\BEAS
\Delta^\top Q \Delta & = & \Delta^\top ( Q \Delta) \leqslant \Omega(\Delta) \Omega^\ast( Q \Delta) \\
& \leqslant &  \Omega(\Delta) ( \Omega^\ast( q)  +  \lambda ) \leqslant \frac{3 \lambda}{2} \Omega(\Delta) 
\mbox{ by optimality of } \hat{w}
\\
\Omega(\Delta)
& \leqslant & 
   \Omega_J(\Delta_J) + \rho ^{-1} 
\Omega^J(\Delta_{J^c})  \\
& \leqslant &  \Omega_J(\Delta_J) ( 3 + \frac{1}{\rho } )
\leqslant  \frac{4}{\rho } \Omega_J(\Delta_J).
\EEAS
This implies that $ { \kappa}  \Omega_J(\Delta_J)^2 \leqslant   \Delta^\top Q \Delta \leqslant 
\frac{6 \lambda}{\rho } \Omega_J(\Delta_J)$, and thus 
$
\Omega_J(\Delta_J) \leqslant \frac{6 \lambda}{\kappa \rho }
$, which leads to the desired result, given the previous inequalities.
 
   \end{proof}
   
   \subsection{Proof of proposition~\ref{prop:proba}}
  \begin{proof}
  We have $\Omega^\ast(z) = \max_{A \in \mathcal{D}_F} \frac{\|  z_A \|_q }{F(A)^{1/q}}$. Thus, from the union bound, we get
 \BEAS
  \mathbb{P}( \Omega^\ast(z) > t )  
 & \leqslant &   \sum_{A \in \mathcal{D}_F} \mathbb{P} ( \|  z_A \|_q^q  >  t^q F(A) ).
 \EEAS 
We can then derive concentration inequalities.
We have $\mathbb{E} \| z_A \|_q \leqslant ( \mathbb{E}  \| z_A\|_q^q )^{1/q}
= ( |A|\mathbb{E}  |\varepsilon|^q)^{1/q}
\leqslant 
 2 |A|^{1/q}  q^{1/2}
$, where $\varepsilon$ is a standard normal random variable. Moreover,
$\| z_A\|_q \leqslant \|z_A \|_2$ for $ q\geqslant 2$, and
$\| z_A\|_q \leqslant |A|^{1/q-1/2}\|z_A \|_2$ for $ q\leqslant 2$. We can thus use the concentration of Lipschitz-continuous functions of Gaussian variables, to get for $p \geqslant 2$ and $u\geqslant 0$,
$$
\mathbb{P}  \big( \| z_A \|_q \geqslant  2 |A|^{1/q} \sqrt{q} + u \big) \leqslant e^{-u^2/2}
.$$
For $p<2$ (i.e., $q>2$), we obtain
$$
\mathbb{P}  \big( \| z_A \|_q \geqslant  2 |A|^{1/q} \sqrt{q} + u \big) \leqslant e^{-u^2 |A|^{1 - 2/q}/2}.
$$
We can also bound the expected norm $\mathbb{E}[ \Omega^\ast(z)] $, as
$$
\mathbb{E}[ \Omega^\ast(z) ]
\leqslant 4 \sqrt{ q \log ( 2 |\mathcal{D}_F|)} \max_{A \in \mathcal{D}_F}  \frac{ |A|^{1/q}}{F(A)^{1/q}}.
$$
Together with
$\Omega^\ast(z) \leqslant \| z\|_2 \max_{A \in \mathcal{D}_F} \frac{|A|^{(1/q-1/2)_+}}{   F(A)^{1/q} }$, 
we get
$$
\mathbb{P} \bigg(
\Omega^\ast(z) \geqslant 4 \sqrt{ q \log ( 2 |\mathcal{D}_F|)} \max_{A \in \mathcal{D}_F}  \frac{ |A|^{1/q}}{F(A)^{1/q}} 
+ u \max_{A \in \mathcal{D}_F} \frac{|A|^{(1/q-1/2)_+}}{   F(A)^{1/q} }
\bigg) \leqslant e^{-u^2/2}.
$$
   \end{proof}

 \end{document}